\documentclass[11pt]{article}

\usepackage[T1]{fontenc}
\usepackage{mathpazo}
\usepackage[margin=1.1in]{geometry}
\usepackage{amsmath, amssymb, amsthm}
\usepackage{mathtools}
\usepackage{graphicx}
\usepackage[font=sf]{caption}
\numberwithin{equation}{section}
\makeatletter
\renewcommand{\tagform@}[1]{%
  \makebox[1sp][l]{\hspace{\dimexpr\marginparsep+0.15in\relax}%
    \scalebox{0.75}[1]{\fontfamily{cmr}\selectfont\itshape\small
      Eq~\ignorespaces#1\unskip\@@italiccorr}}%
}
\makeatother
\usepackage[expansion=false]{microtype}
\usepackage{booktabs}
\usepackage{tikz}
\usepackage[colorlinks=true, linkcolor=black, urlcolor=black]{hyperref}
\usetikzlibrary{arrows.meta, calc, decorations.pathreplacing}

\definecolor{vcol}{HTML}{D85A30}
\definecolor{ucol}{HTML}{1D9E75}
\definecolor{gcol}{HTML}{888780}

\newcommand{\R}{\mathbb{R}}
\newcommand{\T}{^{\mathsf{T}}}
\newcommand{\ip}[2]{\langle #1, #2 \rangle}
\newcommand{\Bip}[2]{\left\langle #1, #2 \right\rangle}
\newcommand{\vv}[1]{\mathbf{v}_{#1}}
\newcommand{\uu}[1]{\mathbf{u}_{#1}}

\newtheoremstyle{plainish}{1em}{1em}{\itshape}{}{\bfseries}{.}{ }{}
\theoremstyle{plainish}
\newtheorem*{claim}{Claim}

\title{\bfseries Singular Value Decomposition\\[0.3em]
\large\mdseries A geometric rediscovery --- where proofs become algorithms}
\author{Paul Agron}
\date{\today}

\begin{document}
\maketitle
\thispagestyle{empty}

\begin{abstract}
\noindent
This article has a hidden agenda. On its surface it is a leisurely geometric
rediscovery of the singular value decomposition --- but its real claim is that
the pretty piece of mathematics it builds turns out to be the machinery quietly
running a great deal of your machine learning. The same construction that
answers an idle question about ellipses is the algorithm behind principal
component analysis, kernel methods, and PageRank; and it is not merely the
\emph{results} that transfer, but the \emph{proofs} themselves, run as
procedures.

The usual introduction to the SVD states the factorization $A = U\Sigma V\T$ and
justifies it by applying the spectral theorem to $A\T A$ --- correct, and
unilluminating: the geometry arrives already compressed into an algebraic
identity, and a powerful theorem is assumed at the outset to reach a result that
is, in the end, about ellipses. Part~I reverses the order, following the path one
might actually take to discover the decomposition. A linear map sends the unit
circle to an ellipse; an ellipse has axes; one asks what input directions map to
them, and finds --- across example after example --- that they are perpendicular.
In the plane this can be \emph{watched}: rotate a frame, track how badly its
images miss being perpendicular, and a sign change forces a configuration where
they do not miss at all. That surviving frame is where the map stretches hardest.
The equivalence generalizes by maximizing the stretch and recursing, singular
values falling out in order --- and the construction turns out to have proved the
spectral theorem rather than assuming it. The classical statement arrives at the
end, as a name for what the geometry already built.

Part~II collects the debt. Each of the constructions from Part~I reappears as a
working tool: the maximize-and-recurse proof becomes the power method and
PageRank; the lemma locating the maximizer becomes the stopping rule of gradient
descent; the duality between $A\T A$ and $AA\T$ becomes the transport at the
heart of kernel PCA. Every connection is stated with its boundary --- what the
decomposition supplies and where some other idea takes over --- because
``the SVD is everywhere'' is less useful than knowing exactly which part of a
method it is responsible for.

The prerequisites are the standard sophomore sequence; worked examples are small
enough to verify by hand, and every numerical and symbolic claim has been
checked.
\end{abstract}

\section*{Orientation}
\addcontentsline{toc}{section}{Orientation}

This is a mathematics article written with a data scientist's temperament. It
does not open with a definition and a theorem; it opens with an experiment ---
run a matrix on the unit circle and look at what comes back --- notices a
pattern, checks whether the pattern survives in higher dimensions, and only then
sets out to prove it. That arc, from observation to conjecture to proof, is not
just how the material is delivered here; it is the lesson. The singular value
decomposition is usually handed to the reader fully formed. This article instead
tries to let it be \emph{discovered}, the way one might actually stumble onto it
by paying attention to a picture.

\paragraph{Tools you will want in hand.} The prerequisites are the standard
sophomore sequence for a mathematics, physics, or data-science student, and
nothing beyond it. From multivariable calculus (Calc III): directional
derivatives and Lagrange multipliers. The one genuinely calculus-flavored step
in the article --- \S\ref{sec:variational}, where we maximize the length of an
image vector over the unit sphere --- is exactly the reasoning behind Lagrange
multipliers, carried out by hand with an explicit curve rather than a multiplier;
a reader who has seen constrained optimization will recognize it even though no
$\lambda$ appears. From a proof-based linear algebra course: inner products and
their bilinearity, orthonormal bases, orthogonal matrices, and
eigenvalues and eigenvectors through the spectral theorem. From single-variable
calculus: the intermediate value theorem, which does the entire work of the
first existence proof (\S\ref{sec:existence}). And throughout, comfort reading a
short proof --- the content prerequisites are light, but past
\S\ref{sec:observe} the article genuinely argues rather than merely computes.

\paragraph{Two signposts.} First, you will meet the spectral theorem twice: once
as a result you already know and are asked to recall, and once, near the end of
the construction (\S\ref{sec:algebraic}), as something the geometry rebuilds from
scratch without assuming it. That doubling is deliberate --- the point is to
watch a theorem you were handed get \emph{earned}. Second, this article is a
stepping stone. The singular value decomposition is the geometric foundation
underneath principal component analysis, low-rank approximation, latent-factor
models, and much of the linear structure in modern machine learning. If you are
heading toward those tools, the picture built here is what they stand on;
\S\ref{sec:pca} makes the first of those connections explicit.

\part{Discovering the decomposition}

\section{Circles become ellipses}
\label{sec:observe}

Fix a matrix $A \in \R^{m \times n}$ and think of it as a map $\R^n \to \R^m$.
Feed it the unit circle --- all input vectors of length one --- and look at where
they land. Take
\begin{equation}A = \begin{pmatrix} 2 & 1 \\ -0.5 & 1.5 \end{pmatrix},
\end{equation}
which will serve as the running example. Applying $A$ to the unit circle produces
the closed curve on the right of Figure~\ref{fig:ellipse}: an \emph{ellipse}. Try
other matrices and the same thing happens --- the circle is stretched and turned,
but it always comes out an ellipse, never a peanut or a teardrop. (For a
degenerate matrix the ellipse can flatten to a segment, and in higher dimensions
the sphere becomes an ellipsoid; \S\ref{sec:ellipsoid} makes this precise. For
now the picture is all we need.)

An ellipse has two distinguished directions: its \emph{axes}, the long one and
the short one, perpendicular to each other. They are features of the output. It
is natural to ask what they came \emph{from} --- what input directions $A$ sent
to them.

\begin{figure}[htbp]
\centering
\begin{tikzpicture}[scale=1.4, >={Stealth[length=2mm]}]
\begin{scope}
  \draw[gcol!50, thin] (-1.35,0) -- (1.35,0);
  \draw[gcol!50, thin] (0,-1.35) -- (0,1.35);
  \draw[gcol, thin] (0,0) circle (1);
  \draw[->, vcol, very thick] (0,0) -- (0.8281,0.5606);
  \draw[->, ucol, very thick] (0,0) -- (-0.5606,0.8281);
  \node[vcol, font=\small] at (1.10,0.70) {$\vv1$};
  \node[ucol, font=\small] at (-0.78,1.02) {$\vv2$};
  \node[gcol, font=\scriptsize] at (0,-1.62) {unit circle};
  \node[gcol, font=\scriptsize] at (0,-1.92) {preimages of the axes};
\end{scope}
\node at (2.15,0) {$\xmapsto{\;\;A\;\;}$};
\begin{scope}[xshift=4.6cm]
  \draw[gcol!50, thin] (-2.45,0) -- (2.45,0);
  \draw[gcol!50, thin] (0,-1.75) -- (0,1.75);
  \draw[gcol!35, thin, dashed] (0,0) circle (1);
  \draw[gcol, thin, rotate=10.90] (0,0) ellipse (2.2575 and 1.5504);
  \draw[->, vcol, very thick] (0,0) -- (2.2168,0.4269);
  \draw[->, ucol, very thick] (0,0) -- (-0.2932,1.5224);
  \node[vcol, font=\small] at (2.44,0.60) {$\sigma_1 \uu1$};
  \node[ucol, font=\small] at (-0.62,1.68) {$\sigma_2 \uu2$};
  \node[gcol, font=\scriptsize] at (0,-2.02) {image ellipse};
  \node[gcol, font=\scriptsize] at (0,-2.32) {axes (long, short)};
\end{scope}
\end{tikzpicture}
\caption{The unit circle maps to an ellipse under
$A = \left(\protect\begin{smallmatrix}2&1\\-0.5&1.5\protect\end{smallmatrix}\right)$.
The ellipse's axes point along $\uu1$ (long, length $\sigma_1$) and $\uu2$
(short, length $\sigma_2$). Their preimages $\vv1, \vv2$ --- the input directions
$A$ sent to the axes --- are marked on the circle. The question of
\S\ref{sec:evidence} is whether those preimages are perpendicular.}
\label{fig:ellipse}
\end{figure}
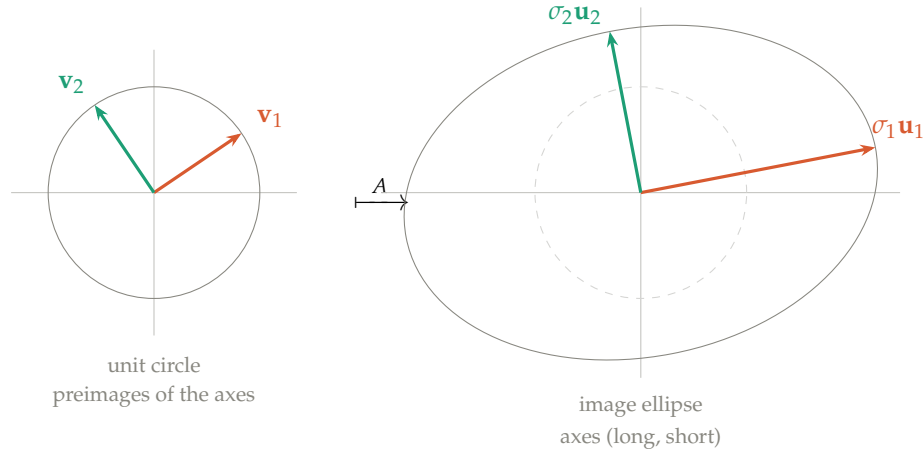

\begin{quote}
\itshape
Are the preimages of the ellipse's axes special?
\end{quote}

This is the question the first half of the article answers. Nothing about the
setup suggests the preimages should be nice: $A$ mangles most directions, tilting
and skewing as it goes. But the axes are where the image is longest and shortest,
so their preimages are at least distinguished by \emph{something}, and it is worth
looking.

\subsection{What the examples show}
\label{sec:evidence}

Compute the axis-preimages for a few matrices and a pattern appears at once.

\begin{center}
\renewcommand{\arraystretch}{1.3}
\begin{tabular}{@{}lcccc@{}}
\toprule
$A$ & axes at & preimages at & axes $\perp$? & preimages $\perp$? \\
\midrule
$\left(\begin{smallmatrix}2&1\\-0.5&1.5\end{smallmatrix}\right)$
  & $10.9^\circ,\,100.9^\circ$ & $34.1^\circ,\,124.1^\circ$ & yes & yes \\
$\left(\begin{smallmatrix}1&1\\0&1\end{smallmatrix}\right)$
  & $31.7^\circ,\,121.7^\circ$ & $58.3^\circ,\,148.3^\circ$ & yes & yes \\
$\left(\begin{smallmatrix}1&2\\2&1\end{smallmatrix}\right)$
  & $45^\circ,\,{-45}^\circ$ & $45^\circ,\,135^\circ$ & yes & yes \\
\bottomrule
\end{tabular}
\end{center}

\noindent
The axes are perpendicular --- of course they are, they are the axes of an
ellipse. The striking column is the last one: \emph{their preimages are
perpendicular too.} The two input directions that $A$ sends to the ellipse's axes
form a right angle, and $A$ carries that right angle to another right angle (the
axes).

This should be surprising, because $A$ mangles right angles as a rule. The shear
$\left(\begin{smallmatrix}1&1\\0&1\end{smallmatrix}\right)$ sends the
perpendicular pair $e_1, e_2$ to $(1,0)$ and $(1,1)$, which meet at $45^\circ$;
run almost any perpendicular pair through almost any matrix and it comes out
skew. Perpendicularity is not something linear maps respect. So a pair that $A$
keeps perpendicular is a genuine exception, not the norm --- which is exactly
what makes ``is there such a pair, and does it always land on the axes?'' a real
question rather than an idle one. Every example above answers \emph{yes}, but
examples are not a proof, and it is not yet clear why it should be true or
whether the perpendicular preimage pair is unique. The next sections settle both,
first in the plane, then in general.

Because this pair is what we are hunting, give its members names now. Write
$\vv1, \vv2$ for the preimages of the long and short axes, and $\sigma_1 \geq
\sigma_2$ for the axis half-lengths, so that
\begin{equation}A\vv1 = \sigma_1 \uu1, \qquad A\vv2 = \sigma_2 \uu2,
\end{equation}
where $\uu1, \uu2$ are the unit vectors along the axes. These are the
\emph{singular values} ($\sigma_i$), \emph{right singular vectors} ($\vv i$), and
\emph{left singular vectors} ($\uu i$) of $A$, and assembling the two relations
into matrix form gives $A = U\Sigma V\T$ --- the singular value decomposition.
\S\ref{sec:naming} returns to the names; the task between here and there is to
prove the frame exists.

\section{Why the perpendicular pair must exist}
\label{sec:existence}

The examples of \S\ref{sec:evidence} suggest that among all perpendicular input
pairs, there is one whose images stay perpendicular --- and that this one lands
on the ellipse's axes. To prove it, forget the ellipse for a moment and hunt
directly for such a pair. Take a perpendicular frame $\vv1(\theta), \vv2(\theta)$
and rotate it rigidly as $\theta$ varies, watching what $A$ does to the right
angle between the arms. Measure the failure with
\begin{equation}f(\theta) \;=\; \ip{A\vv1(\theta)}{A\vv2(\theta)},
\end{equation}
the inner product of the two \emph{images}. The inner product is the natural
gauge of ``how far from perpendicular'': it is zero exactly when the images are
perpendicular, positive when the angle between them has closed below a right
angle, negative when it has opened past one (Figure~\ref{fig:sweep}). So $f$
measures the \emph{damage} $A$ does to the right angle; the input arms stay
perpendicular throughout, by construction, and we are asking whether any $\theta$
makes the damage vanish. If one does, that frame's images are perpendicular ---
and, being the images of a perpendicular pair under $A$, they are exactly the
kind of pair the axes are.

\subsection{The sign flips every quarter turn}
\label{sec:planar}

Rotate the frame by exactly $90^\circ$. What happens to its two arms?

\begin{itemize}\itemsep0.2em
\item $\vv1$ rotated by $90^\circ$ becomes $\vv2$. That is what perpendicular means.
\item $\vv2$ rotated by $90^\circ$ becomes $-\vv1$. Another quarter turn past
$\vv2$ lands \emph{opposite} $\vv1$, not on it.
\end{itemize}

\noindent
So $\vv1(\theta + 90^\circ) = \vv2(\theta)$ and $\vv2(\theta+90^\circ) = -\vv1(\theta)$.
Apply $A$, use linearity to pull the sign out, and use the symmetry of the
inner product:
\begin{equation}f(\theta+90^\circ)
\;=\; \ip{A\vv2}{A(-\vv1)}
\;=\; -\ip{A\vv2}{A\vv1}
\;=\; -\ip{A\vv1}{A\vv2}
\;=\; -f(\theta).
\end{equation}

This is worth pausing on, because the swap alone would achieve nothing --- the
inner product is symmetric, so exchanging its two arguments leaves it
unchanged. \emph{The minus sign is doing all the work.} And that minus sign is
unavoidable: you cannot rotate an orthonormal pair by a quarter turn and
recover the same pair, because one arm always comes back reversed.

\begin{claim}
Every matrix preserves at least one right angle.
\end{claim}

\begin{proof}
$f$ is continuous, and $f(\theta_0 + 90^\circ) = -f(\theta_0)$ for any
$\theta_0$. If $f(\theta_0) = 0$ we are done. Otherwise $f(\theta_0)$ and
$f(\theta_0+90^\circ)$ have opposite signs, so by the intermediate value
theorem $f$ vanishes somewhere strictly between.
\end{proof}

\noindent
The argument used nothing about $A$ --- not invertibility, not squareness, not
any structure. In $\R^n$ the same conclusion holds via the spectral theorem
(\S\ref{sec:algebraic}); the planar version is just the case you can see.

\subsection{How many such angles are there?}
\label{sec:uniqueness-f}

Continuity gives existence but is silent on count. To settle that, compute $f$
in closed form. Writing $A = \left(\begin{smallmatrix} a & b \\ c & d\end{smallmatrix}\right)$
and $\vv1 = (\cos\theta, \sin\theta)$, $\vv2 = (-\sin\theta, \cos\theta)$, expanding
and collecting terms gives
\begin{equation}\boxed{\;
f(\theta) \;=\; (ab + cd)\cos 2\theta \;-\; \tfrac12\left(a^2 - b^2 + c^2 - d^2\right)\sin 2\theta
\;}
\end{equation}
Two features of this expression matter, and the second is the important one.

\emph{First}, everything is in $2\theta$, not $\theta$. That is the algebraic
shadow of the quarter-turn symmetry: advancing $\theta$ by $90^\circ$ advances
$2\theta$ by $180^\circ$, which negates both $\cos$ and $\sin$. The sign flip
proved above is visible directly in the formula.

\emph{Second, and crucially, there is no constant term.} $f$ is a pure
sinusoid oscillating about zero --- not a sinusoid plus an offset. This is why
the sign flip is not merely a curiosity but a guarantee: a sinusoid with a
nonzero mean might never reach zero, but this one must.

Write it in amplitude--phase form as
\begin{equation}f(\theta) = R\cos(2\theta - \varphi), \qquad
R = \sqrt{A_1^2 + A_2^2}, \quad \varphi = \operatorname{atan2}(A_2, A_1),
\end{equation}
where $A_1 = ab+cd$ and $A_2 = -\tfrac12(a^2-b^2+c^2-d^2)$. Suppose $R \neq 0$.
As $\theta$ runs over a half turn $[0^\circ, 180^\circ)$, the argument $2\theta$
runs over a \emph{full} turn, so the cosine vanishes exactly twice, at
\begin{equation}\theta = \frac{\varphi \pm 90^\circ}{2}.
\end{equation}
The two roots are $90^\circ$ apart (Figure~\ref{fig:sweep}). But a frame and the same frame rotated by
$90^\circ$ consist of \emph{the same pair of lines} --- as shown above, the
quarter turn merely swaps the arms and reverses one. So the two roots describe
one geometric object.

\begin{claim}
Unless $R = 0$, the frame is unique: exactly one perpendicular pair of
directions has perpendicular images, up to relabelling the arms and flipping
their signs.
\end{claim}

\paragraph{The degenerate case.} $R = 0$ means $A_1 = A_2 = 0$, i.e.\ $f \equiv 0$
and \emph{every} frame works. The two coefficients are readable off $A\T A$:
\begin{equation}A\T A = \begin{pmatrix} a^2+c^2 & ab+cd \\ ab+cd & b^2+d^2\end{pmatrix},
\end{equation}
so $A_1$ is its off-diagonal entry and $-2A_2$ the difference of its diagonal
entries. Both vanish precisely when $A\T A = \sigma^2 I$ --- when $A$ is a
scalar multiple of an orthogonal matrix. Then $A$ maps the sphere to a sphere,
there is no direction of maximum stretch, and no frame is distinguished. The
generic case and the degenerate case both fall out of the one formula.

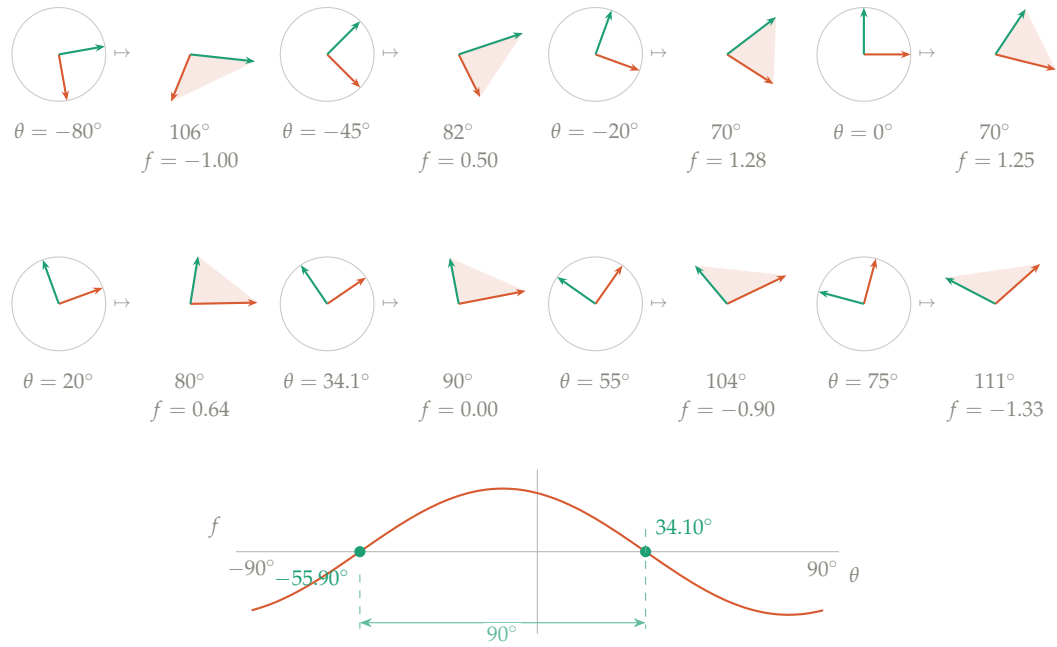
\begin{figure}[htbp]
\centering
\begin{tikzpicture}[scale=1, >={Stealth[length=1.3mm]}]
\draw[gcol!45, thin] (0.780,0.000) circle (0.620);
\draw[->, vcol, thick] (0.780,0.000) -- (0.888,-0.611);
\draw[->, ucol, thick] (0.780,0.000) -- (1.391,0.108);
\node[gcol, font=\tiny] at (1.620,0.000) {$\mapsto$};
\fill[vcol, opacity=0.13] (2.520,0.000) -- (2.265,-0.626) -- (3.377,-0.093) -- cycle;
\draw[->, vcol, thick] (2.520,0.000) -- (2.265,-0.626);
\draw[->, ucol, thick] (2.520,0.000) -- (3.377,-0.093);
\node[gcol, font=\scriptsize] at (0.780,-1.020) {$\theta=-80^\circ$};
\node[gcol, font=\scriptsize] at (2.520,-1.020) {$106^\circ$};
\node[gcol, font=\scriptsize] at (2.520,-1.420) {$f=-1.00$};
\draw[gcol!45, thin] (4.330,0.000) circle (0.620);
\draw[->, vcol, thick] (4.330,0.000) -- (4.768,-0.438);
\draw[->, ucol, thick] (4.330,0.000) -- (4.768,0.438);
\node[gcol, font=\tiny] at (5.170,0.000) {$\mapsto$};
\fill[vcol, opacity=0.13] (6.070,0.000) -- (6.353,-0.566) -- (6.919,0.283) -- cycle;
\draw[->, vcol, thick] (6.070,0.000) -- (6.353,-0.566);
\draw[->, ucol, thick] (6.070,0.000) -- (6.919,0.283);
\node[gcol, font=\scriptsize] at (4.330,-1.020) {$\theta=-45^\circ$};
\node[gcol, font=\scriptsize] at (6.070,-1.020) {$82^\circ$};
\node[gcol, font=\scriptsize] at (6.070,-1.420) {$f=0.50$};
\draw[gcol!45, thin] (7.880,0.000) circle (0.620);
\draw[->, vcol, thick] (7.880,0.000) -- (8.463,-0.212);
\draw[->, ucol, thick] (7.880,0.000) -- (8.092,0.583);
\node[gcol, font=\tiny] at (8.720,0.000) {$\mapsto$};
\fill[vcol, opacity=0.13] (9.620,0.000) -- (10.235,-0.393) -- (10.269,0.495) -- cycle;
\draw[->, vcol, thick] (9.620,0.000) -- (10.235,-0.393);
\draw[->, ucol, thick] (9.620,0.000) -- (10.269,0.495);
\node[gcol, font=\scriptsize] at (7.880,-1.020) {$\theta=-20^\circ$};
\node[gcol, font=\scriptsize] at (9.620,-1.020) {$70^\circ$};
\node[gcol, font=\scriptsize] at (9.620,-1.420) {$f=1.28$};
\draw[gcol!45, thin] (11.430,0.000) circle (0.620);
\draw[->, vcol, thick] (11.430,0.000) -- (12.050,0.000);
\draw[->, ucol, thick] (11.430,0.000) -- (11.430,0.620);
\node[gcol, font=\tiny] at (12.270,0.000) {$\mapsto$};
\fill[vcol, opacity=0.13] (13.170,0.000) -- (13.970,-0.200) -- (13.570,0.600) -- cycle;
\draw[->, vcol, thick] (13.170,0.000) -- (13.970,-0.200);
\draw[->, ucol, thick] (13.170,0.000) -- (13.570,0.600);
\node[gcol, font=\scriptsize] at (11.430,-1.020) {$\theta=0^\circ$};
\node[gcol, font=\scriptsize] at (13.170,-1.020) {$70^\circ$};
\node[gcol, font=\scriptsize] at (13.170,-1.420) {$f=1.25$};
\draw[gcol!45, thin] (0.780,-3.300) circle (0.620);
\draw[->, vcol, thick] (0.780,-3.300) -- (1.363,-3.088);
\draw[->, ucol, thick] (0.780,-3.300) -- (0.568,-2.717);
\node[gcol, font=\tiny] at (1.620,-3.300) {$\mapsto$};
\fill[vcol, opacity=0.13] (2.520,-3.300) -- (3.409,-3.283) -- (2.622,-2.668) -- cycle;
\draw[->, vcol, thick] (2.520,-3.300) -- (3.409,-3.283);
\draw[->, ucol, thick] (2.520,-3.300) -- (2.622,-2.668);
\node[gcol, font=\scriptsize] at (0.780,-4.320) {$\theta=20^\circ$};
\node[gcol, font=\scriptsize] at (2.520,-4.320) {$80^\circ$};
\node[gcol, font=\scriptsize] at (2.520,-4.720) {$f=0.64$};
\draw[gcol!45, thin] (4.330,-3.300) circle (0.620);
\draw[->, vcol, thick] (4.330,-3.300) -- (4.843,-2.952);
\draw[->, ucol, thick] (4.330,-3.300) -- (3.982,-2.787);
\node[gcol, font=\tiny] at (5.170,-3.300) {$\mapsto$};
\fill[vcol, opacity=0.13] (6.070,-3.300) -- (6.957,-3.129) -- (5.953,-2.691) -- cycle;
\draw[->, vcol, thick] (6.070,-3.300) -- (6.957,-3.129);
\draw[->, ucol, thick] (6.070,-3.300) -- (5.953,-2.691);
\node[gcol, font=\scriptsize] at (4.330,-4.320) {$\theta=34.1^\circ$};
\node[gcol, font=\scriptsize] at (6.070,-4.320) {$90^\circ$};
\node[gcol, font=\scriptsize] at (6.070,-4.720) {$f=0.00$};
\draw[gcol!45, thin] (7.880,-3.300) circle (0.620);
\draw[->, vcol, thick] (7.880,-3.300) -- (8.236,-2.792);
\draw[->, ucol, thick] (7.880,-3.300) -- (7.372,-2.944);
\node[gcol, font=\tiny] at (8.720,-3.300) {$\mapsto$};
\fill[vcol, opacity=0.13] (9.620,-3.300) -- (10.407,-2.923) -- (9.194,-2.792) -- cycle;
\draw[->, vcol, thick] (9.620,-3.300) -- (10.407,-2.923);
\draw[->, ucol, thick] (9.620,-3.300) -- (9.194,-2.792);
\node[gcol, font=\scriptsize] at (7.880,-4.320) {$\theta=55^\circ$};
\node[gcol, font=\scriptsize] at (9.620,-4.320) {$104^\circ$};
\node[gcol, font=\scriptsize] at (9.620,-4.720) {$f=-0.90$};
\draw[gcol!45, thin] (11.430,-3.300) circle (0.620);
\draw[->, vcol, thick] (11.430,-3.300) -- (11.590,-2.701);
\draw[->, ucol, thick] (11.430,-3.300) -- (10.831,-3.140);
\node[gcol, font=\tiny] at (12.270,-3.300) {$\mapsto$};
\fill[vcol, opacity=0.13] (13.170,-3.300) -- (13.763,-2.772) -- (12.501,-2.952) -- cycle;
\draw[->, vcol, thick] (13.170,-3.300) -- (13.763,-2.772);
\draw[->, ucol, thick] (13.170,-3.300) -- (12.501,-2.952);
\node[gcol, font=\scriptsize] at (11.430,-4.320) {$\theta=75^\circ$};
\node[gcol, font=\scriptsize] at (13.170,-4.320) {$111^\circ$};
\node[gcol, font=\scriptsize] at (13.170,-4.720) {$f=-1.33$};
\end{tikzpicture}
\vspace{1.2em}

\begin{tikzpicture}[scale=1, >={Stealth[length=1.6mm]}]
\draw[gcol!60, thin] (-3.990,0) -- (3.990,0);
\draw[gcol!60, thin] (0,-1.085) -- (0,1.085);
\draw[vcol, thick] (-3.780,-0.775) -- (-3.738,-0.764) -- (-3.696,-0.751) -- (-3.654,-0.738) -- (-3.612,-0.724) -- (-3.570,-0.709) -- (-3.528,-0.694) -- (-3.486,-0.677) -- (-3.444,-0.660) -- (-3.402,-0.641) -- (-3.360,-0.622) -- (-3.318,-0.602) -- (-3.276,-0.582) -- (-3.234,-0.561) -- (-3.192,-0.539) -- (-3.150,-0.516) -- (-3.108,-0.493) -- (-3.066,-0.469) -- (-3.024,-0.445) -- (-2.982,-0.420) -- (-2.940,-0.394) -- (-2.898,-0.369) -- (-2.856,-0.342) -- (-2.814,-0.315) -- (-2.772,-0.288) -- (-2.730,-0.261) -- (-2.688,-0.233) -- (-2.646,-0.205) -- (-2.604,-0.176) -- (-2.562,-0.148) -- (-2.520,-0.119) -- (-2.478,-0.090) -- (-2.436,-0.061) -- (-2.394,-0.032) -- (-2.352,-0.003) -- (-2.310,0.026) -- (-2.268,0.055) -- (-2.226,0.084) -- (-2.184,0.113) -- (-2.142,0.142) -- (-2.100,0.171) -- (-2.058,0.199) -- (-2.016,0.227) -- (-1.974,0.255) -- (-1.932,0.283) -- (-1.890,0.310) -- (-1.848,0.337) -- (-1.806,0.363) -- (-1.764,0.389) -- (-1.722,0.415) -- (-1.680,0.440) -- (-1.638,0.464) -- (-1.596,0.488) -- (-1.554,0.512) -- (-1.512,0.534) -- (-1.470,0.556) -- (-1.428,0.578) -- (-1.386,0.598) -- (-1.344,0.618) -- (-1.302,0.638) -- (-1.260,0.656) -- (-1.218,0.674) -- (-1.176,0.690) -- (-1.134,0.706) -- (-1.092,0.721) -- (-1.050,0.736) -- (-1.008,0.749) -- (-0.966,0.761) -- (-0.924,0.773) -- (-0.882,0.783) -- (-0.840,0.793) -- (-0.798,0.802) -- (-0.756,0.809) -- (-0.714,0.816) -- (-0.672,0.822) -- (-0.630,0.826) -- (-0.588,0.830) -- (-0.546,0.832) -- (-0.504,0.834) -- (-0.462,0.835) -- (-0.420,0.834) -- (-0.378,0.833) -- (-0.336,0.830) -- (-0.294,0.827) -- (-0.252,0.823) -- (-0.210,0.817) -- (-0.168,0.811) -- (-0.126,0.803) -- (-0.084,0.795) -- (-0.042,0.785) -- (0.000,0.775) -- (0.042,0.764) -- (0.084,0.751) -- (0.126,0.738) -- (0.168,0.724) -- (0.210,0.709) -- (0.252,0.694) -- (0.294,0.677) -- (0.336,0.660) -- (0.378,0.641) -- (0.420,0.622) -- (0.462,0.602) -- (0.504,0.582) -- (0.546,0.561) -- (0.588,0.539) -- (0.630,0.516) -- (0.672,0.493) -- (0.714,0.469) -- (0.756,0.445) -- (0.798,0.420) -- (0.840,0.394) -- (0.882,0.369) -- (0.924,0.342) -- (0.966,0.315) -- (1.008,0.288) -- (1.050,0.261) -- (1.092,0.233) -- (1.134,0.205) -- (1.176,0.176) -- (1.218,0.148) -- (1.260,0.119) -- (1.302,0.090) -- (1.344,0.061) -- (1.386,0.032) -- (1.428,0.003) -- (1.470,-0.026) -- (1.512,-0.055) -- (1.554,-0.084) -- (1.596,-0.113) -- (1.638,-0.142) -- (1.680,-0.171) -- (1.722,-0.199) -- (1.764,-0.227) -- (1.806,-0.255) -- (1.848,-0.283) -- (1.890,-0.310) -- (1.932,-0.337) -- (1.974,-0.363) -- (2.016,-0.389) -- (2.058,-0.415) -- (2.100,-0.440) -- (2.142,-0.464) -- (2.184,-0.488) -- (2.226,-0.512) -- (2.268,-0.534) -- (2.310,-0.556) -- (2.352,-0.578) -- (2.394,-0.598) -- (2.436,-0.618) -- (2.478,-0.638) -- (2.520,-0.656) -- (2.562,-0.674) -- (2.604,-0.690) -- (2.646,-0.706) -- (2.688,-0.721) -- (2.730,-0.736) -- (2.772,-0.749) -- (2.814,-0.761) -- (2.856,-0.773) -- (2.898,-0.783) -- (2.940,-0.793) -- (2.982,-0.802) -- (3.024,-0.809) -- (3.066,-0.816) -- (3.108,-0.822) -- (3.150,-0.826) -- (3.192,-0.830) -- (3.234,-0.832) -- (3.276,-0.834) -- (3.318,-0.835) -- (3.360,-0.834) -- (3.402,-0.833) -- (3.444,-0.830) -- (3.486,-0.827) -- (3.528,-0.823) -- (3.570,-0.817) -- (3.612,-0.811) -- (3.654,-0.803) -- (3.696,-0.795) -- (3.738,-0.785) -- (3.780,-0.775);
\fill[ucol] (1.432,0) circle (2.1pt);
\fill[ucol] (-2.348,0) circle (2.1pt);
\node[ucol, font=\scriptsize, anchor=south west] at (1.432,0.10) {$34.10^\circ$};
\node[ucol, font=\scriptsize, anchor=north east] at (-2.348,-0.10) {$-55.90^\circ$};
\draw[ucol!70, thin, dashed] (-2.348,-0.300) -- (-2.348,-1.020);
\draw[ucol!70, thin, dashed] (1.432,0.300) -- (1.432,-1.020);
\draw[<->, ucol!70, thin] (-2.348,-0.94) -- node[below, font=\scriptsize, fill=white, inner sep=1pt] {$90^\circ$} (1.432,-0.94);
\node[gcol, font=\scriptsize] at (-3.780,-0.22) {$-90^\circ$};
\node[gcol, font=\scriptsize] at (3.780,-0.22) {$90^\circ$};
\node[gcol, font=\scriptsize, anchor=south east] at (-4.05,0.05) {$f$};
\node[gcol, font=\scriptsize, anchor=north west] at (3.990,-0.06) {$\theta$};
\end{tikzpicture}
\caption{\textbf{Sweeping the frame.} Each panel shows an orthonormal input
frame (left, always perpendicular) and its image under
$A = \left(\begin{smallmatrix}2 & 1\\ -0.5 & 1.5\end{smallmatrix}\right)$
(right, generally not). The number below each image is the angle between
$A\vv1$ and $A\vv2$; the shaded wedge makes it visible. As $\theta$ increases
the wedge narrows to $70^\circ$, then widens back through $90^\circ$ --- the
highlighted panel, where the right angle survives --- and on past
$110^\circ$. The plot below tracks $f(\theta)$ over the full half-turn: a pure
sinusoid in $2\theta$ with no vertical offset, crossing zero exactly twice, at
$34.10^\circ$ and $-55.90^\circ$. Those two crossings are $90^\circ$ apart and
describe the same pair of lines.}
\label{fig:sweep}
\end{figure}

\subsection{The same thing happens in three dimensions}
\label{sec:threed}

The plane is settled: a perpendicular pair whose images are perpendicular exists,
is essentially unique, and lands on the ellipse's axes. Before generalizing,
it is worth checking whether the phenomenon even persists --- so run the
experiment one dimension up. In $\R^3$ the unit sphere maps to an ellipsoid,
which has \emph{three} axes, mutually perpendicular. Do their preimages form an
orthonormal triple, as the two axis-preimages did in the plane?

Take
\begin{equation}A = \begin{pmatrix} 1 & 2 & 0 \\ 0 & 1 & 1 \\ 1 & 0 & 1 \end{pmatrix},
\end{equation}
compute the ellipsoid's axes and pull them back. The semi-axis lengths come out
$\sigma_1 = 2.532$, $\sigma_2 = 1.347$, $\sigma_3 = 0.879$, and the three
preimages $\vv1, \vv2, \vv3$ satisfy
\begin{equation}\ip{\vv i}{\vv j} = 0 \quad (i \neq j), \qquad \|\vv i\| = 1 :
\end{equation}
a perfect orthonormal triple, angles $90.00^\circ$ to the numerical precision
checked. Repeating with other $3\times 3$ matrices --- symmetric, asymmetric,
random --- gives the same verdict every time (Figure~\ref{fig:threed}). The
ellipsoid's axes always pull back to a perpendicular frame, and $A$ carries that
frame to the axes: $A\vv i = \sigma_i \uu i$, just as in the plane. And it is
just as surprising: a \emph{generic} orthonormal triple, run through this same
$A$, comes out skewed --- the axis-preimages are again the exception, not the
rule.

\begin{figure}[htbp]
\centering
\begin{tikzpicture}[scale=1,>={Stealth[length=1.6mm]}]
\draw[gcol!40,thin] (0.000,0.000) circle (1.500);
\draw[gcol!25,very thin] (1.359,0.217) -- (1.284,0.265) -- (1.195,0.310) -- (1.092,0.352) -- (0.976,0.390) -- (0.849,0.423) -- (0.713,0.451) -- (0.569,0.475) -- (0.418,0.493) -- (0.263,0.505) -- (0.104,0.512) -- (-0.055,0.513) -- (-0.214,0.508) -- (-0.371,0.497) -- (-0.523,0.481) -- (-0.670,0.459) -- (-0.809,0.432) -- (-0.938,0.400) -- (-1.058,0.364) -- (-1.165,0.323) -- (-1.258,0.279) -- (-1.338,0.232) -- (-1.403,0.182) -- (-1.451,0.130) -- (-1.483,0.076) -- (-1.499,0.022) -- (-1.497,-0.033) -- (-1.478,-0.087) -- (-1.443,-0.140) -- (-1.391,-0.192) -- (-1.324,-0.241) -- (-1.241,-0.288) -- (-1.145,-0.332) -- (-1.035,-0.371) -- (-0.914,-0.407) -- (-0.782,-0.438) -- (-0.642,-0.464) -- (-0.494,-0.484) -- (-0.341,-0.500) -- (-0.184,-0.509) -- (-0.024,-0.513) -- (0.135,-0.511) -- (0.293,-0.503) -- (0.448,-0.490) -- (0.598,-0.471) -- (0.740,-0.446) -- (0.875,-0.417) -- (0.999,-0.383) -- (1.113,-0.344) -- (1.213,-0.302) -- (1.300,-0.256) -- (1.372,-0.207) -- (1.429,-0.156) -- (1.469,-0.103) -- (1.493,-0.049) -- (1.500,0.005) -- (1.490,0.060) -- (1.463,0.114) -- (1.419,0.166) -- (1.359,0.217);
\draw[gcol!25,very thin] (1.359,0.217) -- (1.352,0.365) -- (1.329,0.510) -- (1.291,0.649) -- (1.238,0.780) -- (1.171,0.902) -- (1.091,1.015) -- (0.999,1.115) -- (0.895,1.204) -- (0.781,1.278) -- (0.659,1.338) -- (0.529,1.383) -- (0.392,1.412) -- (0.252,1.425) -- (0.108,1.422) -- (-0.036,1.403) -- (-0.180,1.368) -- (-0.323,1.318) -- (-0.461,1.252) -- (-0.595,1.173) -- (-0.721,1.080) -- (-0.840,0.975) -- (-0.948,0.859) -- (-1.047,0.733) -- (-1.133,0.599) -- (-1.206,0.458) -- (-1.266,0.311) -- (-1.312,0.162) -- (-1.342,0.010) -- (-1.358,-0.141) -- (-1.358,-0.292) -- (-1.342,-0.438) -- (-1.312,-0.580) -- (-1.266,-0.715) -- (-1.206,-0.842) -- (-1.133,-0.960) -- (-1.047,-1.067) -- (-0.948,-1.161) -- (-0.840,-1.243) -- (-0.721,-1.310) -- (-0.595,-1.362) -- (-0.461,-1.400) -- (-0.323,-1.421) -- (-0.180,-1.426) -- (-0.036,-1.415) -- (0.108,-1.388) -- (0.252,-1.345) -- (0.392,-1.287) -- (0.529,-1.214) -- (0.659,-1.128) -- (0.781,-1.029) -- (0.895,-0.918) -- (0.999,-0.797) -- (1.091,-0.667) -- (1.171,-0.529) -- (1.238,-0.385) -- (1.291,-0.237) -- (1.329,-0.086) -- (1.352,0.066) -- (1.359,0.217);
\draw[->,vcol,very thick] (0.000,0.000) -- (0.075,0.903);
\node[vcol,font=\scriptsize] at (0.089,1.066) {$\mathbf{v}_1$};
\draw[->,ucol,very thick] (0.000,0.000) -- (0.683,1.044);
\node[ucol,font=\scriptsize] at (0.806,1.232) {$\mathbf{v}_2$};
\draw[->,gcol,very thick] (0.000,0.000) -- (1.333,-0.586);
\node[gcol,font=\scriptsize] at (1.573,-0.692) {$\mathbf{v}_3$};
\node[gcol,font=\scriptsize] at (0.000,-1.900) {unit sphere};
\node[gcol,font=\scriptsize] at (0.000,-2.220) {preimages $\mathbf{v}_i$ (orthonormal)};
\node at (3.0,0) {$\xmapsto{\;A\;}$};
\draw[gcol!30,thin] (6.498,0.596) -- (6.577,0.628) -- (6.649,0.653) -- (6.713,0.670) -- (6.770,0.679) -- (6.818,0.681) -- (6.856,0.675) -- (6.885,0.662) -- (6.904,0.641) -- (6.912,0.612) -- (6.910,0.577) -- (6.898,0.535) -- (6.876,0.488) -- (6.844,0.434) -- (6.802,0.376) -- (6.751,0.313) -- (6.692,0.247) -- (6.624,0.179) -- (6.550,0.108) -- (6.470,0.036) -- (6.384,-0.037) -- (6.294,-0.109) -- (6.200,-0.180) -- (6.104,-0.249) -- (6.007,-0.315) -- (5.910,-0.377) -- (5.814,-0.435) -- (5.720,-0.489) -- (5.629,-0.536) -- (5.543,-0.578) -- (5.462,-0.613) -- (5.386,-0.641) -- (5.318,-0.662) -- (5.257,-0.675) -- (5.205,-0.681) -- (5.162,-0.679) -- (5.128,-0.669) -- (5.105,-0.652) -- (5.091,-0.627) -- (5.088,-0.596) -- (5.094,-0.557) -- (5.112,-0.512) -- (5.139,-0.462) -- (5.176,-0.406) -- (5.223,-0.345) -- (5.278,-0.281) -- (5.341,-0.213) -- (5.412,-0.143) -- (5.489,-0.072) -- (5.573,0.001) -- (5.661,0.073) -- (5.753,0.145) -- (5.848,0.215) -- (5.944,0.282) -- (6.041,0.347) -- (6.138,0.407) -- (6.233,0.463) -- (6.326,0.513) -- (6.414,0.558) -- (6.498,0.596);
\draw[gcol!30,thin] (6.498,0.596) -- (6.471,0.666) -- (6.438,0.728) -- (6.400,0.782) -- (6.358,0.827) -- (6.311,0.863) -- (6.262,0.889) -- (6.209,0.904) -- (6.153,0.910) -- (6.096,0.905) -- (6.038,0.890) -- (5.980,0.865) -- (5.921,0.830) -- (5.864,0.786) -- (5.808,0.733) -- (5.754,0.671) -- (5.703,0.602) -- (5.656,0.526) -- (5.612,0.444) -- (5.573,0.357) -- (5.539,0.266) -- (5.509,0.172) -- (5.486,0.076) -- (5.468,-0.020) -- (5.456,-0.117) -- (5.451,-0.212) -- (5.451,-0.305) -- (5.458,-0.394) -- (5.471,-0.479) -- (5.490,-0.559) -- (5.515,-0.632) -- (5.545,-0.698) -- (5.580,-0.756) -- (5.620,-0.806) -- (5.665,-0.846) -- (5.713,-0.877) -- (5.765,-0.898) -- (5.819,-0.908) -- (5.875,-0.909) -- (5.933,-0.899) -- (5.991,-0.879) -- (6.050,-0.849) -- (6.107,-0.809) -- (6.164,-0.760) -- (6.219,-0.703) -- (6.271,-0.638) -- (6.321,-0.565) -- (6.366,-0.486) -- (6.408,-0.401) -- (6.445,-0.312) -- (6.477,-0.220) -- (6.503,-0.125) -- (6.524,-0.028) -- (6.539,0.069) -- (6.547,0.165) -- (6.550,0.259) -- (6.546,0.350) -- (6.536,0.438) -- (6.520,0.520) -- (6.498,0.596);
\draw[gcol!30,thin] (6.764,0.329) -- (6.735,0.401) -- (6.698,0.467) -- (6.653,0.529) -- (6.600,0.584) -- (6.541,0.633) -- (6.475,0.674) -- (6.404,0.708) -- (6.329,0.734) -- (6.249,0.752) -- (6.167,0.761) -- (6.083,0.761) -- (5.998,0.753) -- (5.913,0.736) -- (5.829,0.711) -- (5.747,0.678) -- (5.668,0.638) -- (5.593,0.589) -- (5.522,0.535) -- (5.457,0.474) -- (5.397,0.408) -- (5.345,0.337) -- (5.300,0.263) -- (5.263,0.185) -- (5.234,0.105) -- (5.214,0.025) -- (5.203,-0.057) -- (5.201,-0.137) -- (5.208,-0.216) -- (5.224,-0.292) -- (5.249,-0.366) -- (5.282,-0.435) -- (5.324,-0.499) -- (5.373,-0.557) -- (5.429,-0.609) -- (5.491,-0.655) -- (5.560,-0.692) -- (5.633,-0.722) -- (5.711,-0.744) -- (5.792,-0.757) -- (5.875,-0.762) -- (5.959,-0.758) -- (6.044,-0.746) -- (6.129,-0.725) -- (6.212,-0.696) -- (6.293,-0.659) -- (6.370,-0.614) -- (6.443,-0.563) -- (6.511,-0.505) -- (6.574,-0.442) -- (6.630,-0.373) -- (6.678,-0.300) -- (6.719,-0.224) -- (6.752,-0.145) -- (6.777,-0.065) -- (6.792,0.016) -- (6.799,0.097) -- (6.796,0.177) -- (6.785,0.255) -- (6.764,0.329);
\draw[->,vcol,very thick] (6.000,0.000) -- (6.801,0.747);
\node[vcol,font=\scriptsize] at (6.921,0.859) {$\sigma_1\mathbf{u}_1$};
\draw[->,ucol,very thick] (6.000,0.000) -- (5.607,0.607);
\node[ucol,font=\scriptsize] at (5.548,0.698) {$\sigma_2\mathbf{u}_2$};
\draw[->,gcol,very thick] (6.000,0.000) -- (6.301,0.098);
\node[gcol,font=\scriptsize] at (6.346,0.113) {$\sigma_3\mathbf{u}_3$};
\node[gcol,font=\scriptsize] at (6.000,-2.000) {image ellipsoid};
\node[gcol,font=\scriptsize] at (6.000,-2.320) {axes $\sigma_i\mathbf{u}_i$ (orthogonal)};
\end{tikzpicture}
\caption{\textbf{The three-dimensional experiment.} The unit sphere (left) maps
under $A = \left(\protect\begin{smallmatrix} 1 & 2 & 0 \\ 0 & 1 & 1 \\ 1 & 0 & 1\protect\end{smallmatrix}\right)$
to an ellipsoid (right). The ellipsoid's three axes $\sigma_i\uu i$ are mutually
perpendicular; their preimages $\vv i$ --- the input directions $A$ sends to
those axes --- form an orthonormal triple. The phenomenon of \S\ref{sec:evidence}
survives intact into three dimensions.}
\label{fig:threed}
\end{figure}
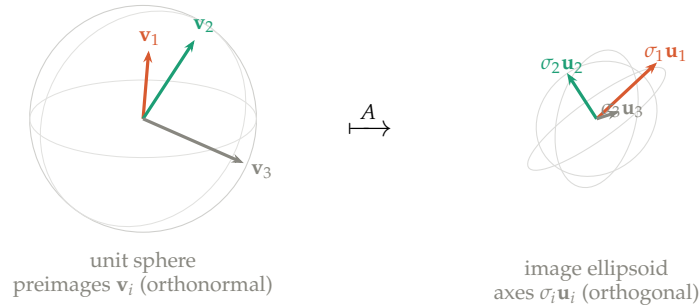

So the pattern is real and dimension-independent, but the \emph{proof} was not.
The rotating-frame argument was thoroughly planar: it swept a single angle and
watched a scalar change sign. In $\R^3$ a frame is not fixed by one angle, and
in $\R^n$ an orthonormal frame has $n(n-1)/2$ parameters while its images must
satisfy $n(n-1)/2$ orthogonality conditions --- a coupled system, not a single
equation, and the intermediate value theorem has nothing to say about it. Fixing
up one pair disturbs another. A new idea is needed to carry the planar success
upward, and the experiments hint at where to find it: the axes came in
\emph{order}, a longest and a shortest, and the longest is where the image
stretches most. Perhaps maximizing the stretch is the handle.

\subsection{From two dimensions to \texorpdfstring{$n$}{n}}
\label{sec:variational}

The plan is to establish, \emph{in the plane where both sides can be seen}, that
the surviving frame and the direction of maximum stretch are the same object;
then to notice that one half of that equivalence never used the dimension; then
to recurse on it, peeling off the longest axis at each step as the experiments
suggest.

\paragraph{One thing the sweep assumes.} Recall that in $f(\theta)$ the input
arms are kept perpendicular by construction --- the domain-side right angle is a
constraint we \emph{impose}, and the whole question is whether the image-side
right angle can be made to hold at the same time. We are not deriving that the
frame is orthonormal; we are searching, among orthonormal input frames, for one
whose images are orthonormal too.

\subsubsection*{In the plane: the crossing is the peak}

So far $f$ has been the only thing measured. But Figure~\ref{fig:sweep} shows
something else changing as $\theta$ turns: the images grow and shrink. At
$\theta = -80^\circ$ the first arm's image is short; by $\theta = 34^\circ$ it
has reached the ellipse's long axis. That second quantity is the one the reader
already cares about --- $\sigma_1$ was defined as a \emph{length}, not as an
angle. So track it:
\begin{equation}g(\theta) \;=\; \|A\vv1(\theta)\|^2 ,
\end{equation}
the squared length of the first arm's image. The square is a convenience:
differentiating $\|A\vv1\|$ itself produces a square root in the denominator,
while the squared version stays polynomial. Nothing is lost, because squaring is
monotone on nonnegative numbers --- whatever $\theta$ maximizes $g$ maximizes
$\|A\vv1\|$ too.

Figure~\ref{fig:stretch} sweeps the same frame again, watching $g$ instead of
$f$. The two crossings of $f$ sit exactly at the peak and the trough of $g$.
That is not a coincidence, and one observation explains it.

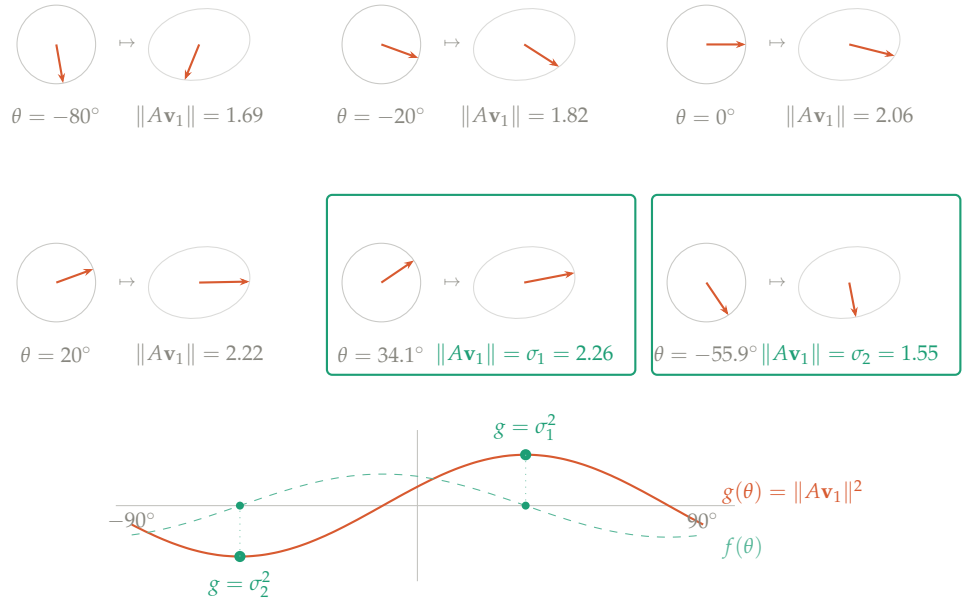
\begin{figure}[htbp]
\centering
\begin{tikzpicture}[scale=1, >={Stealth[length=1.3mm]}]
\draw[gcol!45, thin] (0.660,0.000) circle (0.520);
\draw[->, vcol, thick] (0.660,0.000) -- (0.750,-0.512);
\node[gcol, font=\tiny] at (1.600,0.000) {$\mapsto$};
\draw[gcol!30, thin, rotate around={10.900:(2.550,0.000)}] (2.550,0.000) ellipse (0.677 and 0.465);
\draw[->, vcol, thick] (2.550,0.000) -- (2.359,-0.469);
\node[gcol, font=\scriptsize] at (0.660,-0.950) {$\theta=-80^\circ$};
\node[gcol, font=\scriptsize] at (2.550,-0.950) {$\|A\mathbf{v}_1\|=1.69$};
\draw[gcol!45, thin] (4.960,0.000) circle (0.520);
\draw[->, vcol, thick] (4.960,0.000) -- (5.449,-0.178);
\node[gcol, font=\tiny] at (5.900,0.000) {$\mapsto$};
\draw[gcol!30, thin, rotate around={10.900:(6.850,0.000)}] (6.850,0.000) ellipse (0.677 and 0.465);
\draw[->, vcol, thick] (6.850,0.000) -- (7.311,-0.295);
\node[gcol, font=\scriptsize] at (4.960,-0.950) {$\theta=-20^\circ$};
\node[gcol, font=\scriptsize] at (6.850,-0.950) {$\|A\mathbf{v}_1\|=1.82$};
\draw[gcol!45, thin] (9.260,0.000) circle (0.520);
\draw[->, vcol, thick] (9.260,0.000) -- (9.780,0.000);
\node[gcol, font=\tiny] at (10.200,0.000) {$\mapsto$};
\draw[gcol!30, thin, rotate around={10.900:(11.150,0.000)}] (11.150,0.000) ellipse (0.677 and 0.465);
\draw[->, vcol, thick] (11.150,0.000) -- (11.750,-0.150);
\node[gcol, font=\scriptsize] at (9.260,-0.950) {$\theta=0^\circ$};
\node[gcol, font=\scriptsize] at (11.150,-0.950) {$\|A\mathbf{v}_1\|=2.06$};
\draw[gcol!45, thin] (0.660,-3.150) circle (0.520);
\draw[->, vcol, thick] (0.660,-3.150) -- (1.149,-2.972);
\node[gcol, font=\tiny] at (1.600,-3.150) {$\mapsto$};
\draw[gcol!30, thin, rotate around={10.900:(2.550,-3.150)}] (2.550,-3.150) ellipse (0.677 and 0.465);
\draw[->, vcol, thick] (2.550,-3.150) -- (3.216,-3.137);
\node[gcol, font=\scriptsize] at (0.660,-4.100) {$\theta=20^\circ$};
\node[gcol, font=\scriptsize] at (2.550,-4.100) {$\|A\mathbf{v}_1\|=2.22$};
\draw[ucol, line width=0.8pt, rounded corners=2pt] (4.240,-4.370) rectangle (8.320,-1.990);
\draw[gcol!45, thin] (4.960,-3.150) circle (0.520);
\draw[->, vcol, thick] (4.960,-3.150) -- (5.391,-2.858);
\node[gcol, font=\tiny] at (5.900,-3.150) {$\mapsto$};
\draw[gcol!30, thin, rotate around={10.900:(6.850,-3.150)}] (6.850,-3.150) ellipse (0.677 and 0.465);
\draw[->, vcol, thick] (6.850,-3.150) -- (7.515,-3.022);
\node[gcol, font=\scriptsize] at (4.960,-4.100) {$\theta=34.1^\circ$};
\node[ucol, font=\scriptsize] at (6.850,-4.100) {$\|A\mathbf{v}_1\|=\sigma_1=2.26$};
\draw[ucol, line width=0.8pt, rounded corners=2pt] (8.540,-4.370) rectangle (12.620,-1.990);
\draw[gcol!45, thin] (9.260,-3.150) circle (0.520);
\draw[->, vcol, thick] (9.260,-3.150) -- (9.552,-3.581);
\node[gcol, font=\tiny] at (10.200,-3.150) {$\mapsto$};
\draw[gcol!30, thin, rotate around={10.900:(11.150,-3.150)}] (11.150,-3.150) ellipse (0.677 and 0.465);
\draw[->, vcol, thick] (11.150,-3.150) -- (11.238,-3.607);
\node[gcol, font=\scriptsize] at (9.260,-4.100) {$\theta=-55.9^\circ$};
\node[ucol, font=\scriptsize] at (11.150,-4.100) {$\|A\mathbf{v}_1\|=\sigma_2=1.55$};
\end{tikzpicture}
\vspace{0.9em}

\begin{tikzpicture}[scale=1, >={Stealth[length=1.4mm]}]
\draw[gcol!50, thin] (-3.990,0) -- (4.200,0);
\draw[gcol!50, thin] (0,-1.000) -- (0,1.000);
\draw[vcol, thick] (-3.780,-0.250) -- (-3.738,-0.272) -- (-3.696,-0.293) -- (-3.654,-0.314) -- (-3.612,-0.335) -- (-3.570,-0.355) -- (-3.528,-0.374) -- (-3.486,-0.394) -- (-3.444,-0.413) -- (-3.402,-0.431) -- (-3.360,-0.449) -- (-3.318,-0.466) -- (-3.276,-0.483) -- (-3.234,-0.499) -- (-3.192,-0.514) -- (-3.150,-0.529) -- (-3.108,-0.543) -- (-3.066,-0.557) -- (-3.024,-0.570) -- (-2.982,-0.582) -- (-2.940,-0.593) -- (-2.898,-0.604) -- (-2.856,-0.614) -- (-2.814,-0.623) -- (-2.772,-0.632) -- (-2.730,-0.639) -- (-2.688,-0.646) -- (-2.646,-0.653) -- (-2.604,-0.658) -- (-2.562,-0.663) -- (-2.520,-0.666) -- (-2.478,-0.669) -- (-2.436,-0.671) -- (-2.394,-0.673) -- (-2.352,-0.673) -- (-2.310,-0.673) -- (-2.268,-0.672) -- (-2.226,-0.670) -- (-2.184,-0.667) -- (-2.142,-0.663) -- (-2.100,-0.659) -- (-2.058,-0.654) -- (-2.016,-0.648) -- (-1.974,-0.641) -- (-1.932,-0.633) -- (-1.890,-0.625) -- (-1.848,-0.616) -- (-1.806,-0.606) -- (-1.764,-0.595) -- (-1.722,-0.584) -- (-1.680,-0.572) -- (-1.638,-0.559) -- (-1.596,-0.546) -- (-1.554,-0.532) -- (-1.512,-0.517) -- (-1.470,-0.502) -- (-1.428,-0.486) -- (-1.386,-0.469) -- (-1.344,-0.452) -- (-1.302,-0.434) -- (-1.260,-0.416) -- (-1.218,-0.398) -- (-1.176,-0.378) -- (-1.134,-0.359) -- (-1.092,-0.339) -- (-1.050,-0.318) -- (-1.008,-0.297) -- (-0.966,-0.276) -- (-0.924,-0.254) -- (-0.882,-0.232) -- (-0.840,-0.210) -- (-0.798,-0.188) -- (-0.756,-0.165) -- (-0.714,-0.142) -- (-0.672,-0.119) -- (-0.630,-0.096) -- (-0.588,-0.073) -- (-0.546,-0.049) -- (-0.504,-0.026) -- (-0.462,-0.002) -- (-0.420,0.021) -- (-0.378,0.045) -- (-0.336,0.068) -- (-0.294,0.091) -- (-0.252,0.115) -- (-0.210,0.138) -- (-0.168,0.161) -- (-0.126,0.183) -- (-0.084,0.206) -- (-0.042,0.228) -- (0.000,0.250) -- (0.042,0.272) -- (0.084,0.293) -- (0.126,0.314) -- (0.168,0.335) -- (0.210,0.355) -- (0.252,0.374) -- (0.294,0.394) -- (0.336,0.413) -- (0.378,0.431) -- (0.420,0.449) -- (0.462,0.466) -- (0.504,0.483) -- (0.546,0.499) -- (0.588,0.514) -- (0.630,0.529) -- (0.672,0.543) -- (0.714,0.557) -- (0.756,0.570) -- (0.798,0.582) -- (0.840,0.593) -- (0.882,0.604) -- (0.924,0.614) -- (0.966,0.623) -- (1.008,0.632) -- (1.050,0.639) -- (1.092,0.646) -- (1.134,0.653) -- (1.176,0.658) -- (1.218,0.663) -- (1.260,0.666) -- (1.302,0.669) -- (1.344,0.671) -- (1.386,0.673) -- (1.428,0.673) -- (1.470,0.673) -- (1.512,0.672) -- (1.554,0.670) -- (1.596,0.667) -- (1.638,0.663) -- (1.680,0.659) -- (1.722,0.654) -- (1.764,0.648) -- (1.806,0.641) -- (1.848,0.633) -- (1.890,0.625) -- (1.932,0.616) -- (1.974,0.606) -- (2.016,0.595) -- (2.058,0.584) -- (2.100,0.572) -- (2.142,0.559) -- (2.184,0.546) -- (2.226,0.532) -- (2.268,0.517) -- (2.310,0.502) -- (2.352,0.486) -- (2.394,0.469) -- (2.436,0.452) -- (2.478,0.434) -- (2.520,0.416) -- (2.562,0.398) -- (2.604,0.378) -- (2.646,0.359) -- (2.688,0.339) -- (2.730,0.318) -- (2.772,0.297) -- (2.814,0.276) -- (2.856,0.254) -- (2.898,0.232) -- (2.940,0.210) -- (2.982,0.188) -- (3.024,0.165) -- (3.066,0.142) -- (3.108,0.119) -- (3.150,0.096) -- (3.192,0.073) -- (3.234,0.049) -- (3.276,0.026) -- (3.318,0.002) -- (3.360,-0.021) -- (3.402,-0.045) -- (3.444,-0.068) -- (3.486,-0.091) -- (3.528,-0.115) -- (3.570,-0.138) -- (3.612,-0.161) -- (3.654,-0.183) -- (3.696,-0.206) -- (3.738,-0.228) -- (3.780,-0.250);
\draw[ucol!75, thin, dashed] (-3.780,-0.388) -- (-3.738,-0.382) -- (-3.696,-0.376) -- (-3.654,-0.369) -- (-3.612,-0.362) -- (-3.570,-0.355) -- (-3.528,-0.347) -- (-3.486,-0.338) -- (-3.444,-0.330) -- (-3.402,-0.321) -- (-3.360,-0.311) -- (-3.318,-0.301) -- (-3.276,-0.291) -- (-3.234,-0.280) -- (-3.192,-0.269) -- (-3.150,-0.258) -- (-3.108,-0.246) -- (-3.066,-0.235) -- (-3.024,-0.222) -- (-2.982,-0.210) -- (-2.940,-0.197) -- (-2.898,-0.184) -- (-2.856,-0.171) -- (-2.814,-0.158) -- (-2.772,-0.144) -- (-2.730,-0.130) -- (-2.688,-0.116) -- (-2.646,-0.102) -- (-2.604,-0.088) -- (-2.562,-0.074) -- (-2.520,-0.060) -- (-2.478,-0.045) -- (-2.436,-0.031) -- (-2.394,-0.016) -- (-2.352,-0.001) -- (-2.310,0.013) -- (-2.268,0.028) -- (-2.226,0.042) -- (-2.184,0.057) -- (-2.142,0.071) -- (-2.100,0.085) -- (-2.058,0.100) -- (-2.016,0.114) -- (-1.974,0.128) -- (-1.932,0.141) -- (-1.890,0.155) -- (-1.848,0.168) -- (-1.806,0.182) -- (-1.764,0.195) -- (-1.722,0.207) -- (-1.680,0.220) -- (-1.638,0.232) -- (-1.596,0.244) -- (-1.554,0.256) -- (-1.512,0.267) -- (-1.470,0.278) -- (-1.428,0.289) -- (-1.386,0.299) -- (-1.344,0.309) -- (-1.302,0.319) -- (-1.260,0.328) -- (-1.218,0.337) -- (-1.176,0.345) -- (-1.134,0.353) -- (-1.092,0.361) -- (-1.050,0.368) -- (-1.008,0.374) -- (-0.966,0.381) -- (-0.924,0.386) -- (-0.882,0.392) -- (-0.840,0.396) -- (-0.798,0.401) -- (-0.756,0.405) -- (-0.714,0.408) -- (-0.672,0.411) -- (-0.630,0.413) -- (-0.588,0.415) -- (-0.546,0.416) -- (-0.504,0.417) -- (-0.462,0.417) -- (-0.420,0.417) -- (-0.378,0.416) -- (-0.336,0.415) -- (-0.294,0.413) -- (-0.252,0.411) -- (-0.210,0.409) -- (-0.168,0.405) -- (-0.126,0.402) -- (-0.084,0.397) -- (-0.042,0.393) -- (0.000,0.388) -- (0.042,0.382) -- (0.084,0.376) -- (0.126,0.369) -- (0.168,0.362) -- (0.210,0.355) -- (0.252,0.347) -- (0.294,0.338) -- (0.336,0.330) -- (0.378,0.321) -- (0.420,0.311) -- (0.462,0.301) -- (0.504,0.291) -- (0.546,0.280) -- (0.588,0.269) -- (0.630,0.258) -- (0.672,0.246) -- (0.714,0.235) -- (0.756,0.222) -- (0.798,0.210) -- (0.840,0.197) -- (0.882,0.184) -- (0.924,0.171) -- (0.966,0.158) -- (1.008,0.144) -- (1.050,0.130) -- (1.092,0.116) -- (1.134,0.102) -- (1.176,0.088) -- (1.218,0.074) -- (1.260,0.060) -- (1.302,0.045) -- (1.344,0.031) -- (1.386,0.016) -- (1.428,0.001) -- (1.470,-0.013) -- (1.512,-0.028) -- (1.554,-0.042) -- (1.596,-0.057) -- (1.638,-0.071) -- (1.680,-0.085) -- (1.722,-0.100) -- (1.764,-0.114) -- (1.806,-0.128) -- (1.848,-0.141) -- (1.890,-0.155) -- (1.932,-0.168) -- (1.974,-0.182) -- (2.016,-0.195) -- (2.058,-0.207) -- (2.100,-0.220) -- (2.142,-0.232) -- (2.184,-0.244) -- (2.226,-0.256) -- (2.268,-0.267) -- (2.310,-0.278) -- (2.352,-0.289) -- (2.394,-0.299) -- (2.436,-0.309) -- (2.478,-0.319) -- (2.520,-0.328) -- (2.562,-0.337) -- (2.604,-0.345) -- (2.646,-0.353) -- (2.688,-0.361) -- (2.730,-0.368) -- (2.772,-0.374) -- (2.814,-0.381) -- (2.856,-0.386) -- (2.898,-0.392) -- (2.940,-0.396) -- (2.982,-0.401) -- (3.024,-0.405) -- (3.066,-0.408) -- (3.108,-0.411) -- (3.150,-0.413) -- (3.192,-0.415) -- (3.234,-0.416) -- (3.276,-0.417) -- (3.318,-0.417) -- (3.360,-0.417) -- (3.402,-0.416) -- (3.444,-0.415) -- (3.486,-0.413) -- (3.528,-0.411) -- (3.570,-0.409) -- (3.612,-0.405) -- (3.654,-0.402) -- (3.696,-0.397) -- (3.738,-0.393) -- (3.780,-0.388);
\fill[ucol] (1.432,0.673) circle (2.1pt);
\draw[ucol!55, thin, dotted] (1.432,0.673) -- (1.432,0);
\fill[ucol] (1.432,0) circle (1.5pt);
\fill[ucol] (-2.348,-0.673) circle (2.1pt);
\draw[ucol!55, thin, dotted] (-2.348,-0.673) -- (-2.348,0);
\fill[ucol] (-2.348,0) circle (1.5pt);
\node[ucol, font=\scriptsize, anchor=south] at (1.432,0.733) {$g=\sigma_1^2$};
\node[ucol, font=\scriptsize, anchor=north] at (-2.348,-0.733) {$g=\sigma_2^2$};
\node[vcol, font=\scriptsize, anchor=west] at (3.900,0.180) {$g(\theta)=\|A\mathbf{v}_1\|^2$};
\node[ucol!75, font=\scriptsize, anchor=west] at (3.900,-0.550) {$f(\theta)$};
\node[gcol, font=\scriptsize] at (-3.780,-0.20) {$-90^\circ$};
\node[gcol, font=\scriptsize] at (3.780,-0.20) {$90^\circ$};
\end{tikzpicture}
\caption{\textbf{The same sweep, watching length instead of angle.} Each panel
shows $\vv1$ at angle $\theta$ and its image $A\vv1$, with the image ellipse
faint behind it. Since $\vv1$ is a unit vector its image always lands
\emph{on} the ellipse; what changes is how far around, and so how long the arrow
is. The first five panels climb --- $1.69, 1.82, 2.06, 2.22$ --- until at
$\theta = 34.10^\circ$ the arrow reaches the long axis and $\|A\vv1\| = \sigma_1$.
The last panel jumps to $\theta = -55.90^\circ$, the other extremum, where the
arrow sits on the short axis and $\|A\vv1\| = \sigma_2$. Below,
$g(\theta) = \|A\vv1\|^2$ (solid) is plotted against $f(\theta)$ (dashed, the
same curve as in Figure~\ref{fig:sweep}). The zeros of $f$ fall exactly at the
peak and trough of $g$ --- which is what $dg/d\theta = 2f$ says: $f$ is the
slope of $g$.}
\label{fig:stretch}
\end{figure}

\paragraph{A tool: differentiating an inner product.} The derivation below, and
the one in the converse argument, both rest on a single calculus fact --- the
product rule for inner products. For any two vector-valued functions $u(t)$ and
$w(t)$,
\begin{equation}
\frac{d}{dt}\ip{u(t)}{w(t)} = \ip{u'(t)}{w(t)} + \ip{u(t)}{w'(t)} ,
\end{equation}
which is the ordinary product rule applied coordinatewise and summed, using the
bilinearity of the inner product. Two special cases are all we need. When the two
slots hold the same function, symmetry of the inner product merges the terms:
\begin{equation}
\frac{d}{dt}\ip{u}{u} = 2\ip{u'}{u} .
\end{equation}
And when a constant matrix $A$ sits inside, it commutes with the derivative,
since differentiation acts entrywise and $A$ does not depend on $t$:
$\frac{d}{dt}(Au) = A u'$.

\paragraph{The identity.} As $\theta$ increases, which way does $\vv1$ move? It
rotates --- and rotating $\vv1$ infinitesimally pushes it toward $\vv2$, which
is a quarter turn ahead. In symbols, with $\vv1 = (\cos\theta, \sin\theta)$,
\begin{equation}\frac{d \vv1}{d\theta} = (-\sin\theta,\, \cos\theta) = \vv2 .
\end{equation}
This is the fact that drove the sign flip in \S\ref{sec:planar}, seen
infinitesimally rather than a quarter turn at a time. Now differentiate
$g = \ip{A\vv1}{A\vv1}$, taking the steps one at a time:
\begin{align}
\frac{dg}{d\theta}
  &= \frac{d}{d\theta}\ip{A\vv1}{A\vv1}
     && \text{definition of } g \nonumber\\
  &= 2\Bip{\tfrac{d}{d\theta}(A\vv1)}{A\vv1}
     && \text{same-slot product rule} \nonumber\\
  &= 2\Bip{A\tfrac{d\vv1}{d\theta}}{A\vv1}
     && \text{$A$ constant, so it passes the derivative} \nonumber\\
  &= 2\ip{A\vv2}{A\vv1}
     && \text{using } \tfrac{d\vv1}{d\theta} = \vv2 \nonumber\\
  &= 2 f(\theta)
     && \text{definition of } f .
\end{align}
So
\begin{equation}\boxed{\;\frac{dg}{d\theta} \;=\; 2 f(\theta)\;}
\end{equation}
The damage function is \emph{half the derivative of the stretch}. Nothing was
computed: the identity holds because differentiating the first arm produces the
second, and pairing the second against the first is what $f$ already measures.

This settles what the sweep was showing. The zero crossing of
Figure~\ref{fig:sweep} is not merely correlated with the ellipse's axis --- it
\emph{is} the critical point of $g$. Where the right angle survives, the stretch
is stationary; where the stretch peaks, the right angle survives. The sinusoid
plotted beneath the film strip is, up to a factor of two, the derivative of the
stretch curve in Figure~\ref{fig:stretch}.

That is the equivalence, and both directions are worth having explicitly.

\paragraph{Maximality $\Rightarrow$ orthogonal images.} This is the direction
that will carry the construction into higher dimensions, so it is worth stating
carefully. The claim: \emph{if $\vv1$ is the input direction whose image is
longest, then $A\vv1$ is automatically orthogonal to $Aw$ for every $w$
perpendicular to $\vv1$.} Perpendicularity of the images is not arranged; it is
a consequence of maximizing a length.

The mechanism is the first-derivative test, applied on the sphere. We want to
maximize $\|Av\|$ subject to $\|v\| = 1$, so we cannot vary $v$ freely --- any
admissible variation must keep $v$ on the unit sphere. Fix a maximizer $\vv1$
and a direction $w \perp \vv1$ to test, and move along the arc
\begin{equation}v(t) \;=\; \cos t\, \vv1 + \sin t \, w ,
\end{equation}
chosen for two reasons: it satisfies $v(0) = \vv1$, so $t = 0$ is the point we
care about, and it stays on the sphere, since
$\|v(t)\|^2 = \cos^2 t + \sin^2 t = 1$ (the cross term drops because
$\ip{\vv1}{w} = 0$). As $t$ moves away from $0$ it tips $\vv1$ toward $w$; so
$\frac{d}{dt}$ at $t=0$ is the directional derivative of the objective along
$w$, and \emph{that} is the quantity a maximum forces to vanish.

Restrict the objective to the arc and expand by bilinearity of the inner
product:
\begin{equation}\|Av(t)\|^2
= \|A\vv1\|^2\cos^2 t
\;+\; 2\,\ip{A\vv1}{Aw}\,\sin t\cos t
\;+\; \|Aw\|^2 \sin^2 t .
\end{equation}
Differentiating, the $\cos^2 t$ and $\sin^2 t$ terms contribute
$\mp \sin 2t$, which vanish at $t = 0$; the cross term contributes
$2\ip{A\vv1}{Aw}\cos 2t$, equal to $2\ip{A\vv1}{Aw}$ there. Hence
\begin{equation}\left.\frac{d}{dt}\right|_{t=0} \|Av(t)\|^2 \;=\; 2\,\ip{A\vv1}{Aw} .
\end{equation}
Since $\vv1$ maximizes the objective over the sphere and the arc lies on the
sphere with $v(0) = \vv1$, the function $t \mapsto \|Av(t)\|^2$ has an interior
maximum at $t = 0$, so this derivative is zero:
\begin{equation}\ip{A\vv1}{Aw} = 0 \qquad \text{for every } w \perp \vv1 .
\end{equation}

What the algebra encodes is geometric. The derivative of the stretch as you tip
$\vv1$ toward $w$ is exactly (twice) the overlap $\ip{A\vv1}{Aw}$. If that
overlap were nonzero, the stretch would have nonzero slope in the $w$ direction,
and a small tip would lengthen the image --- impossible, because $\vv1$ was
chosen to make it longest. So maximality \emph{is} vanishing overlap in every
perpendicular direction, which is to say $A\vv1$ is orthogonal to the image of
the entire hyperplane $\vv1^\perp$. The second-derivative test is consistent and
needs no separate check: $\left.\frac{d^2}{dt^2}\right|_0 = 2(\|Aw\|^2 - \|A\vv1\|^2) \leq 0$
precisely because $\|A\vv1\|$ is the largest stretch there is.

\paragraph{Orthogonal images $\Rightarrow$ maximality.} The converse closes the
loop. Its claim: \emph{take the surviving frame --- a perpendicular input pair
$\vv1, \vv2$ whose images are also perpendicular --- and the arm with the longer
image is the direction of maximum stretch.} Where the previous direction started
from a maximizer and produced orthogonality, this one starts from orthogonality
and recovers the maximizer.

Note what is assumed. We take $\vv1, \vv2$ orthonormal in the domain (that is
the frame we impose) \emph{and} their images orthogonal, $\ip{A\vv1}{A\vv2} = 0$
(that is the surviving property). Label the arms so that the first has the
longer image, and write $\sigma_1 = \|A\vv1\| \geq \sigma_2 = \|A\vv2\|$.

Because $\vv1, \vv2$ are orthonormal, they are a basis of the plane, so every
unit input vector is a combination
\begin{equation}v = c_1 \vv1 + c_2 \vv2, \qquad c_1^2 + c_2^2 = 1 ,
\end{equation}
and its image length is governed by how the weight is split. Expanding
$\|Av\|^2$ and using the orthogonality of the images to kill the cross term:
\begin{equation}\|Av\|^2 = \|c_1 A\vv1 + c_2 A\vv2\|^2
= c_1^2 \|A\vv1\|^2 + \underbrace{2c_1c_2\ip{A\vv1}{A\vv2}}_{=\,0} + c_2^2\|A\vv2\|^2
= c_1^2\sigma_1^2 + c_2^2\sigma_2^2 .
\end{equation}
The result is a weighted average of $\sigma_1^2$ and $\sigma_2^2$, with weights
$c_1^2, c_2^2$ that sum to one. An average of two numbers cannot exceed the
larger, and reaches it only by placing all the weight there: $c_1^2 = 1$,
$c_2^2 = 0$, that is $v = \pm\vv1$. So among all unit inputs the longest image
belongs to $\vv1$ --- it is the maximizer.

This is why the two planar descriptions are one. The surviving frame's long arm
$\vv1$ is where the stretch peaks, and $\sigma_1$, defined a moment ago merely as
``the longer of the two image lengths,'' is in fact the largest image length
there is --- the major semi-axis of the ellipse. The step that made it work,
$\|Av\|^2 = c_1^2\sigma_1^2 + c_2^2\sigma_2^2$, is worth remembering: once the
images are orthogonal, the map acts on the $\vv1,\vv2$ coordinates by simple
scaling, with no interference between the arms. That decoupling is the whole
point of the frame, and it is the $2\times 2$ shadow of $A = U\Sigma V\T$.

\subsubsection*{The hinge}

Now look back at the first of those two arguments and ask what it used. A unit
vector $\vv1$; a perpendicular unit vector $w$; the curve
$\cos t\,\vv1 + \sin t\,w$; one derivative. \emph{Nothing in it mentions the
dimension.} In $\R^2$ the only choice of $w$ is $\pm\vv2$, so the conclusion
looks like a statement about a single pair. In $\R^n$ the same curve lies on the
unit sphere $S^{n-1}$, the same derivative computation goes through unchanged,
and the conclusion becomes considerably stronger:
\begin{equation}\ip{A\vv1}{Aw} = 0 \quad\text{for \emph{every} } w \perp \vv1 ,
\end{equation}
which is $n-1$ orthogonality relations at once, not one. The maximizer's image
is orthogonal to the image of the entire hyperplane $\vv1^\perp$.

The other direction does not need to make the trip. It was there to show that
the two planar descriptions coincide; having established that, we keep only the
half that travels.

\subsubsection*{In \texorpdfstring{$\R^n$}{R\^{}n}: recursion}

The construction now runs by induction, and each step is the hinge applied once.

\begin{enumerate}\itemsep0.35em
\item The function $v \mapsto \|Av\|$ is continuous and $S^{n-1}$ is compact, so
a maximizer exists. Call it $\vv1$ and set $\sigma_1 = \|A\vv1\|$.
\item By the hinge, $\ip{A\vv1}{Aw} = 0$ for every $w \perp \vv1$. So
\emph{whatever} we choose next from $\vv1^\perp$, its image is already
orthogonal to $A\vv1$ --- for free, with no further work.
\item Restrict $A$ to the hyperplane $\vv1^\perp \cong \R^{n-1}$ and repeat:
maximize $\|Av\|$ over the unit sphere of $\vv1^\perp$ to get $\vv2$, with
$\sigma_2 = \|A\vv2\| \leq \sigma_1$. Applying the hinge \emph{within} that
hyperplane gives $\ip{A\vv2}{Aw} = 0$ for all $w \in \vv1^\perp$ with
$w \perp \vv2$.
\item Continue. After $n$ steps the frame $\vv1, \dots, \vv n$ is orthonormal by
construction, and its images are pairwise orthogonal: for $i < j$, step $i$'s
hinge applies because $\vv j$ lies in $\vv i^\perp$.
\end{enumerate}

\noindent
Two features of this deserve comment. First, the $\sigma_i$ come out
\emph{sorted}, $\sigma_1 \geq \sigma_2 \geq \dots$, automatically: each step
maximizes over a subset of the previous step's domain, so it cannot do better.
The descending convention is not a convention here --- it is what the
construction produces. Second, step 2 is doing all the work, and it is doing it
\emph{before} the next vector is chosen: maximality of $\vv1$ pre-emptively
guarantees orthogonality against every future choice. That is what makes the
recursion compose, and it is precisely what the planar picture could not show,
having no room to recurse into.

\subsubsection*{What the construction has already proved}
\label{sec:bridge}

Something has quietly happened. Return to the hinge, the statement that at the
maximizer
\begin{equation}\ip{A\vv1}{Aw} = 0 \qquad\text{for every } w \perp \vv1,
\end{equation}
and move $A$ across the inner product onto the other argument:
\begin{equation}\ip{A\vv1}{Aw}
\;=\; (A\vv1)\T (Aw)
\;=\; \vv1\T A\T A\, w
\;=\; \ip{A\T A\, \vv1}{\,w} .
\end{equation}
So the hinge says: \emph{the vector $A\T A \vv1$ is orthogonal to every $w$ in
the hyperplane $\vv1^\perp$.} But the only vectors orthogonal to all of
$\vv1^\perp$ are the multiples of $\vv1$ itself. Therefore
\begin{equation}\boxed{\;A\T A\, \vv1 \;=\; \lambda_1 \vv1\;}
\end{equation}
for some scalar $\lambda_1$ --- and pairing both sides with $\vv1$ identifies it,
\begin{equation}\lambda_1 = \ip{A\T A \vv1}{\vv1} = \ip{A\vv1}{A\vv1} = \|A\vv1\|^2 = \sigma_1^2 .
\end{equation}

\emph{The maximizer is an eigenvector of $A\T A$, with eigenvalue the squared
stretch.} This was not assumed. It fell out of maximality alone, by one
application of the adjoint move.

And the same holds at every step of the recursion, since each $\vv i$ maximizes
over a subspace on which the identical argument runs. So the construction of the
previous page did not merely find a convenient frame --- it produced an
orthonormal basis $\vv1, \dots, \vv n$ of eigenvectors of $A\T A$, with
nonnegative eigenvalues $\sigma_i^2$, sorted. That is a complete
diagonalization:
\begin{equation}A\T A = V \Sigma^2 V\T .
\end{equation}

\noindent
Which is to say: \emph{the variational\footnote{``Variational'' here is used in
the linear-algebra sense: an object found by making a quantity extremal ---
here, maximizing $\|Av\|$ over the unit sphere. This is the finite-dimensional
cousin of the \emph{calculus of variations} proper, which extremizes
\emph{functionals} over infinite-dimensional spaces of functions and produces
Euler--Lagrange differential equations. None of that machinery appears here; our
optimization lives on ordinary spheres in $\R^n$ and is settled by the
first-derivative test (equivalently, Lagrange multipliers).} route has already
proved the spectral theorem for $A\T A$.} The next section does not introduce new
machinery. It names this result in the standard language, states it for symmetric
matrices in general rather than for Gram matrices in particular, and notes that
once one has it, the entire construction collapses to a single sentence.

\paragraph{One caveat.} The route above is specific to matrices of the form
$A\T A$, because it maximizes $\|Ax\|^2$, a quantity that is automatically
nonnegative. A general symmetric $S$ can have negative eigenvalues --- for
instance $\left(\begin{smallmatrix} 2 & 1 \\ 1 & -3\end{smallmatrix}\right)$ has
eigenvalues $-3.19$ and $2.19$ --- so no maximization of a squared length will
ever produce them. The fix is to run the same argument on the \emph{Rayleigh
quotient}
\begin{equation}\rho(x) = \frac{x\T S x}{x\T x},
\end{equation}
whose stationary points are exactly the eigenvectors of $S$ and whose stationary
values are the eigenvalues. The proof is the same derivative computation; only
the functional changes. So the spectral theorem in full generality is this
section's argument with $\|Ax\|^2$ replaced by $\rho(x)$, and
$\rho(x) = \|Ax\|^2$ on the unit sphere when $S = A\T A$ recovers what we did.

\subsection{The same fact, algebraically}
\label{sec:algebraic}

The previous section arrived at $A\T A \vv i = \sigma_i^2 \vv i$ with the
$\vv i$ orthonormal. That statement has a name, and stating it for a general
symmetric matrix rather than for $A\T A$ costs nothing --- the hypothesis that
actually does the work is symmetry, not the particular form of the matrix.

\begin{quote}
\textbf{Spectral theorem.} \itshape Let $S$ be a real symmetric $n \times n$
matrix, $S = S\T$. Then $S$ has $n$ real eigenvalues $\lambda_1, \dots,
\lambda_n$ (with multiplicity), and $\R^n$ admits an \emph{orthonormal} basis
$\mathbf{w}_1, \dots, \mathbf{w}_n$ of eigenvectors, $S\mathbf{w}_i = \lambda_i
\mathbf{w}_i$. Equivalently $S = W \Lambda W\T$ with $W$ orthogonal and
$\Lambda$ diagonal.
\end{quote}

\noindent
Two clauses matter here and neither is free. \emph{Real} eigenvalues: a general
real matrix can have complex ones, as a rotation does. \emph{Orthonormal}
eigenbasis: a general matrix's eigenvectors need not be perpendicular, and a
defective matrix does not have a full set of them at all. Symmetry buys both.

Taking the theorem as given --- it is standard, and \S\ref{sec:bridge} sketched
its proof --- the whole of \S\ref{sec:existence} and \S\ref{sec:variational}
collapses to three lines, which is the point of this section. The rotating frame
and the greedy maximization were two ways of finding by hand what the following
computation supplies at once, in any dimension.

\subsubsection*{Where the symmetry enters}

The orthogonality half is short enough to see in full. Let $\mathbf{w}_i,
\mathbf{w}_j$ be eigenvectors with $\lambda_i \neq \lambda_j$. Start from the
number $\ip{S\mathbf{w}_i}{\mathbf{w}_j}$ and evaluate it two ways.

Pushing $S$ onto the left argument:
\begin{equation}\ip{S\mathbf{w}_i}{\mathbf{w}_j}
\;=\; \ip{\lambda_i \mathbf{w}_i}{\mathbf{w}_j}
\;=\; \lambda_i \ip{\mathbf{w}_i}{\mathbf{w}_j}
\qquad\text{(eigenvector equation, then linearity).}
\end{equation}
Moving $S$ across the inner product to the right argument instead:
\begin{equation}\ip{S\mathbf{w}_i}{\mathbf{w}_j}
\;=\; (S\mathbf{w}_i)\T \mathbf{w}_j
\;=\; \mathbf{w}_i\T S\T \mathbf{w}_j
\;=\; \underbrace{\mathbf{w}_i\T S \mathbf{w}_j}_{\text{here } S\T = S}
\;=\; \ip{\mathbf{w}_i}{S\mathbf{w}_j}
\;=\; \lambda_j \ip{\mathbf{w}_i}{\mathbf{w}_j}.
\end{equation}
The symmetry hypothesis is used exactly once, at the brace: without $S\T = S$
the chain breaks and nothing follows. Setting the two evaluations equal,
\begin{equation}\lambda_i \ip{\mathbf{w}_i}{\mathbf{w}_j} = \lambda_j \ip{\mathbf{w}_i}{\mathbf{w}_j}
\quad\Longrightarrow\quad
(\lambda_i - \lambda_j)\,\ip{\mathbf{w}_i}{\mathbf{w}_j} = 0 .
\end{equation}
Since $\lambda_i \neq \lambda_j$ by assumption, the first factor is nonzero, so
$\ip{\mathbf{w}_i}{\mathbf{w}_j} = 0$: the eigenvectors are perpendicular. (If
$\lambda_i = \lambda_j$ the argument says nothing, and correctly so --- within a
repeated eigenvalue's eigenspace any basis is an eigenbasis, so one simply
chooses an orthonormal one. This is the same freedom catalogued in
\S\ref{sec:uniqueness} for repeated singular values.)

\subsubsection*{Applying it to \texorpdfstring{$A\T A$}{A transpose A}}

First, $A\T A$ qualifies. It is symmetric because transposing reverses the
order of a product and undoes itself:
\begin{equation}(A\T A)\T \;=\; A\T \left(A\T\right)\T \;=\; A\T A .
\end{equation}
The matrix $A\T A$ is called the \textbf{Gram matrix} of $A$: its $(i,j)$ entry
is the inner product of column $i$ of $A$ with column $j$, so its diagonal
holds the squared column lengths and its off-diagonal entries measure how far
from perpendicular the columns are. (Throughout, $(M)_{ij}$ denotes the entry of
$M$ in row $i$, column $j$.)

So the spectral theorem applies and supplies an orthonormal eigenbasis
$\vv1, \dots, \vv n$ with $A\T A \vv i = \lambda_i \vv i$ --- for free, in any
dimension, with no rotating frame required. Now take any two of them, $i \neq j$,
and compute the inner product of their images one step at a time:
\begin{align}
\ip{A\vv i}{A\vv j}
  &= (A\vv i)\T (A\vv j)          && \text{definition of the inner product} \nonumber \\
  &= \vv i\T A\T A \vv j          && \text{transpose of a product} \nonumber \\
  &= \vv i\T (A\T A \vv j)        && \text{regrouping; matrix product is associative} \nonumber \\
  &= \vv i\T (\lambda_j \vv j)    && \text{$\vv j$ is an eigenvector of $A\T A$} \nonumber \\
  &= \lambda_j \,\vv i\T \vv j    && \text{scalars pull out} \nonumber \\
  &= \lambda_j \ip{\vv i}{\vv j}  && \text{definition again} \nonumber \\
  &= 0                            && \text{the $\vv i$ are orthonormal, $i \neq j$.}
\end{align}
The images are orthogonal automatically. \emph{The eigenvectors of $A\T A$ are
the frame the whole note is about}, and the rotating-frame picture of
\S\ref{sec:existence} is what this computation looks like in the plane.

Notice which step did the work. The eigenvector equation converted the
\emph{operator} $A\T A$ into the \emph{scalar} $\lambda_j$, after which
orthogonality of the $\vv i$ finished the job. Feed in a non-eigenvector and
line four fails: $A\T A \vv j$ is some other vector, not a multiple of $\vv j$,
and the inner product has no reason to vanish.

\subsubsection*{Why the singular values are real}

$A\T A$ is not merely symmetric but positive semidefinite. For any vector $x$,
\begin{equation}x\T (A\T A)\, x
\;=\; \underbrace{(x\T A\T)}_{= \,(Ax)\T} (A x)
\;=\; (Ax)\T (Ax)
\;=\; \|Ax\|^2
\;\geq\; 0 ,
\end{equation}
the last step because a squared length is never negative. Now take $x = \vv i$,
an eigenvector:
\begin{equation}\vv i\T (A\T A) \vv i
= \vv i\T (\lambda_i \vv i)
= \lambda_i \underbrace{\vv i\T \vv i}_{=\,\|\vv i\|^2 =\, 1}
= \lambda_i ,
\end{equation}
and the same quantity equals $\|A\vv i\|^2$ by the display above. Hence
\begin{equation}\lambda_i = \|A \vv i\|^2 \geq 0,
\end{equation}
so $\sigma_i := \sqrt{\lambda_i}$ is a well-defined nonnegative real number, and
it is exactly the stretch factor $\|A\vv i\|$ we wanted.

This step deserves its own paragraph because it is easy to skip past. Symmetry
alone gives \emph{real} eigenvalues, but real numbers can be negative, and
$\sqrt{\lambda_i}$ would then be imaginary --- $\Sigma$ would not exist. What
rules that out is positive semidefiniteness, which comes from $A\T A$ being a
Gram matrix rather than from its symmetry. Both properties are needed, and they
are different properties.

\subsubsection*{The same statement in the plane}

The two arguments are not merely analogous; they are the same equation. To see
it, take any symmetric $2\times 2$ matrix $G = \left(\begin{smallmatrix} p & q \\ q & r\end{smallmatrix}\right)$
and ask what its entries become in a frame rotated by $\theta$. Conjugating by
the rotation $R_\theta = \left(\begin{smallmatrix} \cos\theta & -\sin\theta \\ \sin\theta & \cos\theta \end{smallmatrix}\right)$
and reading off the off-diagonal entry gives
\begin{equation}\left(R_\theta\T\, G\, R_\theta\right)_{12}
\;=\; q\,(\cos^2\theta - \sin^2\theta) + (r - p)\sin\theta\cos\theta
\;=\; q \cos 2\theta + \frac{r-p}{2}\,\sin 2\theta .
\end{equation}
Now substitute the Gram matrix, $G = A\T A$, whose entries are
\begin{equation}p = a^2 + c^2, \qquad q = ab + cd, \qquad r = b^2 + d^2 .
\end{equation}
Then $q = ab+cd$ and $\tfrac{r-p}{2} = -\tfrac12(a^2 - b^2 + c^2 - d^2)$, so
\begin{equation}\left(R_\theta\T\, A\T A\, R_\theta\right)_{12}
\;=\; (ab+cd)\cos 2\theta - \tfrac12\left(a^2 - b^2 + c^2 - d^2\right)\sin 2\theta
\;=\; f(\theta),
\end{equation}
matching the closed form of \S\ref{sec:uniqueness-f} term for term. The
off-diagonal entry of the Gram matrix in the rotated frame \emph{is} the damage
function.

So diagonalizing $A\T A$ and solving $f(\theta) = 0$ are not two routes to the
same answer --- they are one problem stated twice. Both ask for the rotation
that annihilates the off-diagonal entry. And the diagonal entries left behind,
\begin{equation}\left(R_\theta\T\, A\T A\, R_\theta\right)_{11} = \|A\vv1\|^2 = \sigma_1^2,
\qquad
\left(R_\theta\T\, A\T A\, R_\theta\right)_{22} = \|A\vv2\|^2 = \sigma_2^2,
\end{equation}
are the squared singular values. For the running example this is checked
numerically in \S\ref{sec:example}.

\subsubsection*{Three routes, one frame}

\begin{center}
\begin{tabular}{@{}p{2.9cm}p{4.6cm}p{4.6cm}@{}}
\toprule
& \textbf{Shows} & \textbf{Costs} \\
\midrule
Rotating frame \newline (\S\ref{sec:existence})
  & why a crossing is forced, and where it sits: the fulcrum where the
    distortion cancels
  & planar only; no route to $\R^n$ \\[0.4em]
Variational \newline (\S\ref{sec:variational})
  & constructive; explains why the $\sigma_i$ come out sorted; proves the
    spectral theorem rather than assuming it
  & no picture above $\R^2$; stationarity replaces balance \\[0.4em]
Spectral \newline (\S\ref{sec:algebraic})
  & fastest; every dimension at once
  & most opaque; the geometry is invisible in the statement \\
\bottomrule
\end{tabular}
\end{center}

\noindent
These are not three competing theorems. They are one theorem approached from
three distances. The variational route is the spectral theorem's proof with the
geometry left in; the spectral statement is that proof with the geometry
compressed out; and the rotating frame is what both look like in the one
dimension where you can watch them happen. The planar picture is worth keeping
even though it does not scale, because it shows what the machinery above it is
\emph{for}: the greedy maximization and the eigenvectors are both, in the end,
hunting the right angle that survives.

\subsection{The same frame from the other side}
\label{sec:duality}

Everything so far was built from the \emph{input} side. We rotated a frame in
the domain, watched its images, and chose the frame $V$ whose images come out
perpendicular --- which forced $V$ to diagonalize $A\T A$. But the picture that
started the article was symmetric: a circle on the left, an ellipse on the
right, and we could just as well have begun on the right. This subsection does
that, and the reward is not a new theorem but a clearer view of the one we have.

Begin, then, from the output. The ellipse's axes are a perpendicular frame
$\uu1, \uu2$ in the codomain --- take that as given, it is what ``axes'' means.
What are their preimages? Solving $A\vv i = \sigma_i \uu i$, the question is
whether the $\vv i$ come out orthonormal. They do, and the reason is the mirror
of the argument we already ran.

\begin{figure}[htbp]
\centering
\begin{tikzpicture}[scale=1,>={Stealth[length=1.7mm]}]
\draw[gcol!50,thin] (-1.55,0)--(1.55,0);
\draw[gcol!50,thin] (0,-1.55)--(0,1.55);
\draw[gcol,thin] (0.00,0.00) circle (1.35);
\draw[->,vcol,very thick] (0,0)--(1.118,0.757);
\draw[->,ucol,very thick] (0,0)--(-0.757,1.118);
\node[vcol,font=\small] at (1.364,0.923) {$\mathbf{v}_1$};
\node[ucol,font=\small] at (-0.946,1.397) {$\mathbf{v}_2$};
\node[gcol,font=\scriptsize] at (0,-1.85) {input sphere};
\node[font=\scriptsize,text width=3.3cm,align=center] at (0,-2.35) {maximize $\|A\mathbf{v}\|$:\\ eigenvectors of $A^{\!\top}\!A$};
\node at (2.5,0) {$\xmapsto{\;\;A\;\;}$};
\draw[gcol!50,thin] (3.60,0)--(7.00,0);
\draw[gcol!50,thin] (5.30,-1.70)--(5.30,1.70);
\draw[gcol!35,thin,dashed] (5.30,0.00) circle (0.84);
\draw[gcol,thin,rotate around={10.90:(5.30,0.00)}] (5.30,0.00) ellipse (1.400 and 0.961);
\draw[->,vcol,very thick] (5.30,0.00)--(6.674,0.265);
\draw[->,ucol,very thick] (5.30,0.00)--(5.118,0.944);
\node[vcol,font=\small] at (6.867,0.352) {$\sigma_1\mathbf{u}_1$};
\node[ucol,font=\small] at (5.064,1.227) {$\sigma_2\mathbf{u}_2$};
\draw[gcol,thin] (6.517,0.234)--(6.487,0.392)--(6.644,0.422);
\node[gcol,font=\scriptsize] at (5.30,-1.90) {image ellipsoid};
\node[font=\scriptsize,text width=3.6cm,align=center] at (5.30,-2.40) {extremal radius $\perp$ surface:\\ eigenvectors of $AA^{\!\top}$};
\end{tikzpicture}
\caption{\textbf{Two proofs, two ends of the map.} \emph{Left:} the input-side
construction of \S\S\ref{sec:existence}--\ref{sec:algebraic} --- maximize
$\|A\mathbf{v}\|$ over the sphere, landing on the eigenvectors $\vv i$ of
$A\T A$. \emph{Right:} the output-side construction --- the ellipsoid's axes are
its extremal radii, where the radius meets the surface at a right angle (marked),
and these are the eigenvectors $\uu i$ of $AA\T$. The arrow $A$ links the two:
$A\vv i = \sigma_i \uu i$. Same variational instinct --- extremize a length ---
applied at opposite ends.}
\label{fig:dual-geo}
\end{figure}
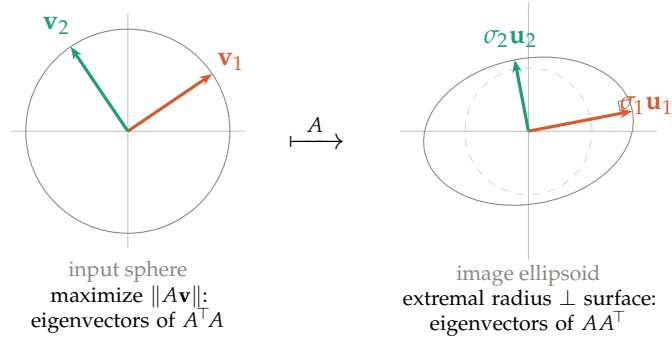

\paragraph{Why the axes are eigenvectors of $AA\T$.} An axis of the ellipsoid is
a direction of extremal radius, and at the tip of an axis the boundary is
perpendicular to the radius (Figure~\ref{fig:dual-geo}, right) --- slide along
the surface anywhere else and the distance to the center changes to first order,
so only there is the length stationary. Written in the output variable
$\mathbf{y} = A\mathbf{x}$, the constraint $\|\mathbf{x}\| = 1$ becomes
$\mathbf{y}\T (AA\T)^{-1} \mathbf{y} = 1$: the ellipsoid is the level set of the
quadratic form $(AA\T)^{-1}$, and the axes of such a form are its eigenvectors.
So the ellipsoid's axes are the eigenvectors of $AA\T$ --- the output-space
companion of $A\T A$.

\paragraph{The mirror computation.} With the $\uu i$ eigenvectors of $AA\T$, the
preimages $\vv i = A^{-1}(\sigma_i \uu i)$ satisfy, for $i \neq j$,
\begin{equation}\ip{\vv i}{\vv j}
  = \sigma_i \sigma_j \, \ip{A^{-1}\uu i}{A^{-1}\uu j}
  = \sigma_i \sigma_j \, \uu i\T (AA\T)^{-1} \uu j
  = \frac{\sigma_i \sigma_j}{\lambda_j}\, \ip{\uu i}{\uu j}
  = 0,
\end{equation}
using $(A^{-1})\T A^{-1} = (AA\T)^{-1}$ and $(AA\T)^{-1}\uu j = \lambda_j^{-1}\uu j$
with $\lambda_j = \sigma_j^2$; the final zero is orthogonality of distinct
eigenvectors of the symmetric matrix $AA\T$. The same substitution with $i = j$
gives $\|\vv i\| = 1$. \emph{The preimages of the axes are orthonormal} ---
proved, this time, without a single derivative, by borrowing the spectral
theorem in the output space rather than building it in the input space.\footnote{%
Written for invertible $A$, so that $A^{-1}$ makes sense and the ellipsoid is
full-dimensional. For rectangular or singular $A$ the same argument runs with
the pseudoinverse on the range of $A$.}

\paragraph{One decomposition, two spectra.} The two proofs are mirror images
across the map (Figure~\ref{fig:dual-square}). The input-side proof diagonalizes
$A\T A$, an $n \times n$ matrix living in the domain, and produces $V$; the
output-side proof diagonalizes $AA\T$, an $m \times m$ matrix living in the
codomain, and produces $U$. These two matrices are different --- different
sizes, even --- but they share their nonzero eigenvalues, the squared singular
values $\sigma_i^2$, because both measure the same stretches. The singular value
decomposition is precisely the statement that their eigenframes are
\emph{compatible}: not two unrelated diagonalizations, but two ends of a single
map welded together by $A\vv i = \sigma_i \uu i$.

\begin{figure}[htbp]
\centering
\begin{tikzpicture}[scale=1,>={Stealth[length=1.8mm]},
  node distance=0pt,
  box/.style={draw=gcol!55,rounded corners=2pt,inner sep=6pt,font=\small,align=center,fill=gcol!5}]
\node[box] (is) at (0,2.4)   {input sphere\\ $\{\,\mathbf{x}:\|\mathbf{x}\|=1\,\}$};
\node[box] (oe) at (6.6,2.4) {output ellipsoid\\ $\{\,A\mathbf{x}\,\}$};
\node[box] (vi) at (0,0)     {$\mathbf{v}_i$: eigenvectors\\ of $A^{\!\top}\!A$ \;(input)};
\node[box] (ui) at (6.6,0)   {$\mathbf{u}_i$: eigenvectors\\ of $AA^{\!\top}$ \;(output)};
\draw[->,vcol,thick] (is) -- node[above,font=\scriptsize,black]{$A$} (oe);
\draw[->,gcol,thick] (is) -- node[left,font=\scriptsize,black,align=right]{maximize\\ $\|A\mathbf{v}\|$}  (vi);
\draw[->,gcol,thick] (oe) -- node[right,font=\scriptsize,black,align=left]{extremal\\ radius}  (ui);
\draw[<->,ucol,thick] (vi) -- node[below,font=\scriptsize,black]{$A\mathbf{v}_i=\sigma_i\mathbf{u}_i$} (ui);
\node[font=\scriptsize,vcol] at (-1.35,1.2) {proof 1};
\node[font=\scriptsize,ucol] at (7.95,1.2) {proof 2};
\end{tikzpicture}
\caption{\textbf{The two proofs as one square.} Down the left, the input-side
construction lands on the eigenvectors of $A\T A$; down the right, the
output-side construction lands on the eigenvectors of $AA\T$. Across the top,
$A$ carries sphere to ellipsoid; across the bottom, $A\vv i = \sigma_i \uu i$
identifies the two frames as the same decomposition seen from opposite corners.
The SVD is the square commuting.}
\label{fig:dual-square}
\end{figure}
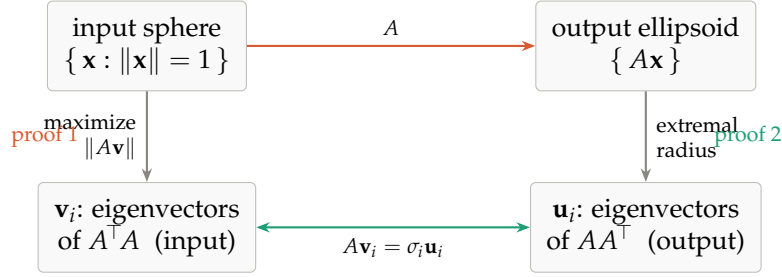

\noindent
This is why the article kept only one of the two: they are the same content, and
the input side is where the frame can be \emph{discovered} rather than merely
verified --- you can rotate it and watch the crossing happen. The output side,
elegant as it is, assumes the axes are already in hand. Kept together, they say
that whichever end you grip the map by, the same orthonormal frame is waiting.

\section{Naming the pieces}
\label{sec:naming}

Write $\vv1, \dots, \vv n$ for that special input frame --- the \textbf{right
singular vectors}. Their images are orthogonal but not unit length, so
factor out the lengths:
\begin{equation}A \vv i \;=\; \sigma_i \uu i, \qquad \sigma_i = \|A\vv i\| \geq 0, \qquad \|\uu i\| = 1 .
\end{equation}
The $\sigma_i$ are the \textbf{singular values} and the $\uu i$ are the
\textbf{left singular vectors}.

Before collecting these into a formula, it is worth seeing what they say about
$A$ as a map. The $\vv i$ are an orthonormal basis of the input space, so any
input $\mathbf{x}$ has an expansion $\mathbf{x} = c_1 \vv1 + \dots + c_n \vv n$,
where $c_i = \ip{\mathbf{x}}{\vv i}$ is simply the coordinate of $\mathbf{x}$
along $\vv i$. Apply $A$ and use linearity together with $A\vv i = \sigma_i\uu i$:
\begin{equation}A\mathbf{x} = c_1 (A\vv1) + \dots + c_n (A\vv n)
  = \sigma_1 c_1 \uu1 + \dots + \sigma_n c_n \uu n .
\end{equation}
Read that off slowly. $A$ takes the coordinate of $\mathbf{x}$ along each input
axis $\vv i$, scales it by $\sigma_i$, and lays the result down along the
matching output axis $\uu i$. The coordinates that go in are the coordinates that
come out --- only the axis labels change, $\vv i \mapsto \uu i$, and each is
stretched by its own $\sigma_i$. In these two bases the map is \emph{diagonal}:
whatever tangle $A$ shows in the standard basis, in the right input frame and the
right output frame it does nothing more complicated than scale each axis. Bundle
this action into matrices --- $V\T$ reads out the input coordinates, $\Sigma$
scales them, $U$ writes them along the output axes --- and it \emph{is} a
factorization:
\begin{equation}\boxed{\; A = U \Sigma V\T \;}
\end{equation}
The construction of \S\S\ref{sec:existence}--\ref{sec:algebraic} has, without
announcing it, proved that such a frame always exists --- so the factorization is
not a lucky feature of some matrices but a theorem about all of them, worth
stating plainly.

\begin{quote}
\textbf{Singular value decomposition.} \itshape Every real $m \times n$ matrix
$A$ factors as
\[
A = U\Sigma V\T,
\]
with $U \in \R^{m\times m}$ and $V \in \R^{n\times n}$ orthogonal, and
$\Sigma \in \R^{m\times n}$ ``diagonal'' --- $\Sigma_{ii} = \sigma_i \geq 0$ and
all other entries zero. The columns of $V$ are the right singular vectors, the
columns of $U$ the left singular vectors, and the $\sigma_i$ the singular
values. Taking $\sigma_1 \geq \sigma_2 \geq \dots \geq 0$, the singular values
are uniquely determined by $A$; the vectors $U, V$ are determined up to the
freedoms catalogued in \S\ref{sec:uniqueness}.
\end{quote}

\noindent
The names \emph{left} and \emph{right} are positional --- $U$ sits on the left
of the product, $V$ on the right --- and carry no deeper meaning.

Note the asymmetry in how the two frames arise. \emph{$V$ is chosen; $U$ is a
consequence.} We select the input frame by a property it has (its images are
orthogonal), and the output frame is then whatever we happen to land on.

\subsection*{Four ways to say the same thing}

The right singular vectors admit several characterizations, all equivalent:

\begin{description}\itemsep0.4em
\item[Geometric.] The orthonormal input frame whose images stay orthogonal.
Every other orthonormal frame comes out skewed.
\item[Metric.] The frame that lands on the principal axes of the image
ellipsoid (\S\ref{sec:ellipsoid}), with $\sigma_i$ the semi-axis lengths.
\item[Variational.] The frame obtained by greedy maximization:
$\vv1 = \arg\max_{\|v\|=1} \|Av\|$, then $\vv2$ the same among unit vectors
orthogonal to $\vv1$, and so on.
\item[Algebraic.] The eigenvectors of $A\T A$.
\end{description}

The first is the one to keep in your head; the last is the one you compute
with. Symmetrically, the left singular vectors are the eigenvectors of $AA\T$.
Since $A\T A$ and $AA\T$ share the same nonzero eigenvalues $\sigma_i^2$, one
$\Sigma$ serves both frames.

\section{The ellipsoid}
\label{sec:ellipsoid}

Let $S^{n-1} = \{v \in \R^n : \|v\| = 1\}$ be the unit sphere in the input
space. Its image $A(S^{n-1})$ is an ellipsoid. To see this, write $v$ in the
$V$-basis as $v = \sum c_i \vv i$ with $\sum c_i^2 = 1$. Then
$Av = \sum c_i \sigma_i \uu i$, so the image point has coordinate
$y_i = \sigma_i c_i$ along $\uu i$, and the constraint becomes
\begin{equation}\sum_i \frac{y_i^2}{\sigma_i^2} \;=\; 1 ,
\end{equation}
the standard ellipsoid equation in $u$-coordinates. The $\uu i$ are the axis
\emph{directions}; the $\sigma_i$ are the axis \emph{lengths}. Neither alone
describes the ellipsoid --- you need both.

This gives the reading of the factorization as a sequence of three motions.
Right to left in $U\Sigma V\T$:

\begin{center}
\begin{tabular}{@{}ll@{}}
\toprule
$V\T$ & rotate the input so that $\vv i$ lands on the $i$th coordinate axis \\
$\Sigma$ & stretch axis $i$ by $\sigma_i$ \\
$U$ & rotate so that the coordinate axes land on the $\uu i$ \\
\bottomrule
\end{tabular}
\end{center}

\noindent
\emph{Every linear map is a rotation, then an axis-aligned scaling, then
another rotation.} A shear does not look like that --- it looks like sliding
--- but it is one, and the SVD is the change of viewpoint that reveals it.

\subsection*{Degenerate cases}

The word ``ellipsoid'' must be read loosely. If some $\sigma_i = 0$ the
ellipsoid is flattened in that direction. If $m > n$ it is a lower-dimensional
ellipsoid embedded in a larger space. If $m < n$ the sphere maps onto the
\emph{solid} ellipsoid, not just its surface, because vectors with a kernel
component land strictly inside.

The statement that covers every case at once: for $A$ of rank $r$,
\begin{quote}
$A$ maps the unit sphere of the \emph{row space} bijectively onto an
$(r{-}1)$-dimensional ellipsoid surface in the \emph{column space}, with
semi-axes $\sigma_1, \dots, \sigma_r$ along $\uu1, \dots, \uu r$. The kernel is
crushed to zero; the cokernel is never reached.
\end{quote}

This is the four-fundamental-subspaces picture with metric information
attached. The SVD says $A$ is an isomorphism from row space to column space,
and diagonalizes it. Explicitly, with $r = \operatorname{rank} A$:
\begin{equation}\underbrace{\vv1 \dots \vv r}_{\text{row space}}, \quad
\underbrace{\vv{r+1} \dots \vv n}_{\ker A}, \quad
\underbrace{\uu1 \dots \uu r}_{\text{column space}}, \quad
\underbrace{\uu{r+1} \dots \uu m}_{\operatorname{coker} A} .
\end{equation}

\section{A worked example}
\label{sec:example}

Take
\begin{equation}A = \begin{pmatrix} 2 & 1 \\ -0.5 & 1.5 \end{pmatrix},
\end{equation}
a matrix with no special structure --- not symmetric, not a multiple of a
rotation. Its Gram matrix is
\begin{equation}A\T A = \begin{pmatrix} 4.25 & 1.25 \\ 1.25 & 3.25 \end{pmatrix},
\end{equation}
whose off-diagonal entry is nonzero, so the standard basis is \emph{not} the
frame we want. Feeding $a = 2$, $b = 1$, $c = -0.5$, $d = 1.5$ into the closed
form of \S\ref{sec:uniqueness-f},
\begin{equation}f(\theta) = 1.25\cos 2\theta - 0.5 \sin 2\theta,
\end{equation}
with amplitude $R = 1.3463$ and phase $\varphi = -21.80^\circ$. The roots
$\theta = (\varphi \pm 90^\circ)/2$ land at $34.10^\circ$ and $-55.90^\circ$ ---
the crossings marked in Figure~\ref{fig:sweep}. Taking the first,
\begin{equation}\vv1 = (0.8281,\, 0.5606), \qquad \vv2 = (-0.5606,\, 0.8281),
\end{equation}
and the images are
\begin{equation}A\vv1 = (2.2168,\, 0.4269), \qquad A\vv2 = (-0.2932,\, 1.5224).
\end{equation}
Their inner product is $(2.2168)(-0.2932) + (0.4269)(1.5224) = 0$, as promised.
Their lengths are the singular values,
\begin{equation}\sigma_1 = 2.2575, \qquad \sigma_2 = 1.5504,
\end{equation}
and normalizing gives $\uu1 = (0.9820, 0.1891)$ and $\uu2 = (-0.1891, 0.9820)$.
As a check, $\sigma_1\sigma_2 = 3.5 = \det A$.

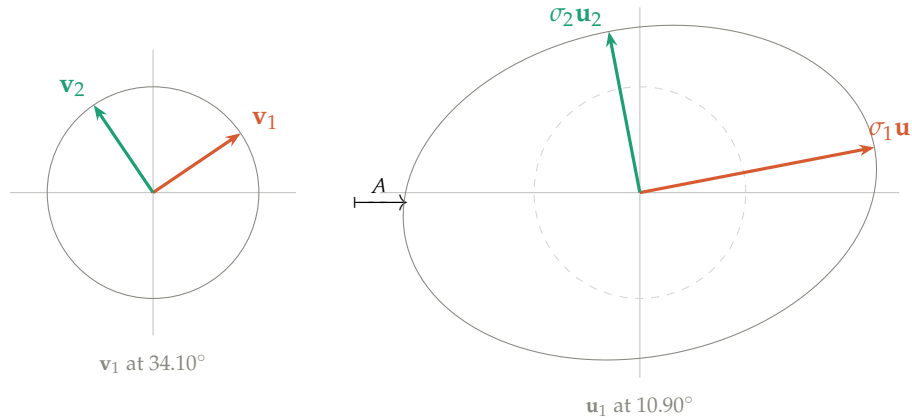
\begin{figure}[h]
\centering
\begin{tikzpicture}[scale=1.4, >={Stealth[length=2mm]}]

\begin{scope}
  \draw[gcol!50, thin] (-1.35,0) -- (1.35,0);
  \draw[gcol!50, thin] (0,-1.35) -- (0,1.35);
  \draw[gcol, thin] (0,0) circle (1);
  \draw[->, vcol, very thick] (0,0) -- (0.8281,0.5606);
  \draw[->, ucol, very thick] (0,0) -- (-0.5606,0.8281);
  \node[vcol, font=\small] at (1.06,0.68) {$\vv1$};
  \node[ucol, font=\small] at (-0.76,0.99) {$\vv2$};
  \node[gcol, font=\scriptsize] at (0,-1.62) {$\vv1$ at $34.10^\circ$};
\end{scope}

\node at (2.15,0) {$\xmapsto{\;\;A\;\;}$};

\begin{scope}[xshift=4.6cm]
  \draw[gcol!50, thin] (-2.45,0) -- (2.45,0);
  \draw[gcol!50, thin] (0,-1.75) -- (0,1.75);
  \draw[gcol!35, thin, dashed] (0,0) circle (1);
  \draw[gcol, thin, rotate=10.90] (0,0) ellipse (2.2575 and 1.5504);
  \draw[->, vcol, very thick] (0,0) -- (2.2168,0.4269);
  \draw[->, ucol, very thick] (0,0) -- (-0.2932,1.5224);
  \node[vcol, font=\small] at (2.42,0.58) {$\sigma_1 \uu1$};
  \node[ucol, font=\small] at (-0.60,1.66) {$\sigma_2 \uu2$};
  \node[gcol, font=\scriptsize] at (0,-2.02) {$\uu1$ at $10.90^\circ$};
\end{scope}

\end{tikzpicture}
\caption{The unit circle and its image. The frame $\vv1, \vv2$ at $34.10^\circ$
is the one whose images stay perpendicular; those images land on the ellipse's
axes, with lengths $\sigma_1 = 2.2575$ and $\sigma_2 = 1.5504$. The input frame
sits at $34.10^\circ$ and the output frame at $10.90^\circ$: a gap of
$23.20^\circ$. The two panels are different copies of $\R^2$; the only link
between them is the pairing $\vv i \mapsto \sigma_i \uu i$.}
\label{fig:ellipse-detail}
\end{figure}

That gap of $23.20^\circ$ between the frames is the rotation $UV\T$. Since
$\det(UV\T) = +1$ it is a genuine rotation, not a reflection. For a generic
matrix it is nonzero and bears no relation to anything else about $A$: knowing
$V$ tells you nothing about $U$ without passing through $A$.

\subsection{An aside: the shear}

The virtue of an arbitrary matrix is that it has no structure to mislead. The
vice is that nobody has any prior expectation about it, so the SVD's claim ---
\emph{this is a rotation, a stretch, and a rotation} --- is unsurprising.

For that, the shear is better:
\begin{equation}A = \begin{pmatrix} 1 & 1 \\ 0 & 1\end{pmatrix} .
\end{equation}
It slides every point rightward by an amount equal to its height. Sliding does
not look like rotate--stretch--rotate. But
\begin{equation}A\T A = \begin{pmatrix} 1 & 1 \\ 1 & 2\end{pmatrix},
\qquad \lambda = \frac{3\pm\sqrt5}{2},
\qquad \sigma_1 = 1.6180, \quad \sigma_2 = 0.6180 ,
\end{equation}
the golden ratio and its reciprocal, with $\sigma_1\sigma_2 = 1 = \det A$. The
frames sit at $\vv1$: $58.28^\circ$ and $\uu1$: $31.72^\circ$, differing by a
rotation of $26.57^\circ$. So the shear \emph{is} a rotate--stretch--rotate,
and the SVD is the change of viewpoint that reveals it.

Why $58.28^\circ$? The shear pushes points rightward in proportion to height.
A frame sitting in the first quadrant has both arms up high, so both get pushed
right and they rotate toward each other: the angle closes, $f > 0$. A frame
straddling the vertical has one arm high-right and one high-left; pushing both
rightward separates them and the angle opens, $f < 0$. At $58.28^\circ$ the two
effects balance. The right angle sits at the fulcrum and comes through intact.

\begin{table}[h]
\centering
\begin{tabular}{@{}lrrrrr@{}}
\toprule
& $\sigma_1$ & $\sigma_2$ & $\vv1$ angle & $\uu1$ angle & gap $UV\T$ \\
\midrule
Generic $\left(\begin{smallmatrix}2&1\\-0.5&1.5\end{smallmatrix}\right)$
  & $2.2575$ & $1.5504$ & $34.10^\circ$ & $10.90^\circ$ & $23.20^\circ$ \\
Shear $\left(\begin{smallmatrix}1&1\\0&1\end{smallmatrix}\right)$
  & $1.6180$ & $0.6180$ & $58.28^\circ$ & $31.72^\circ$ & $26.57^\circ$ \\
\bottomrule
\end{tabular}
\caption{Neither matrix has aligned frames; both rotate by roughly $25^\circ$
between input and output. Alignment $U = V$ would require $A$ symmetric
positive definite (\S\ref{sec:eigen}), and neither matrix is symmetric.}
\end{table}

\noindent
One caution about the shear, which is why it is an aside rather than the main
example. Its $\uu1 = (0.8507, 0.5257)$ is exactly $\vv1 = (0.5257, 0.8507)$
with the coordinates swapped, and the two angles sum to $90^\circ$. That is an
artifact: the shear is \emph{persymmetric} (symmetric about its anti-diagonal),
which forces the relation. It is easy to mistake this tidiness for alignment,
or for a general fact. It is neither. The generic matrix above has no such
pattern, and it is the honest picture.

\section{Can you choose the frame? Uniqueness and freedom}
\label{sec:uniqueness}
\label{sec:choose}

A natural question tests how much the construction really pins down: given $A$,
could you pick \emph{any} orthonormal basis $V$ of the input space and find a
matching orthonormal $U$ and some diagonal $\Sigma'$ with $A = U \Sigma' V\T$?
Perhaps the freedom to choose $\Sigma'$ buys enough slack.

No. Requiring $AV = U\Sigma'$ says precisely that $A\vv1, \dots, A\vv n$ are
mutually orthogonal --- and the $\sigma_i'$ are then \emph{forced} to be the
norms $\|A\vv i\|$, not chosen. The freedom in $\Sigma'$ is illusory; the
orthogonality of the images is the entire constraint, and it is the same
constraint $f(\theta) = 0$ from \S\ref{sec:uniqueness-f}. In the plane this is
exactly what Figure~\ref{fig:sweep} shows: of the continuum of frames, two
crossings, describing one.

The count confirms it in any dimension, and the same count will measure the
leftover freedom in a moment, so it is worth doing once. An orthonormal frame
$V$ in $\R^n$ has $n(n-1)/2$ degrees of freedom --- build it column by column:
$\vv1$ is any unit vector ($n-1$ parameters), $\vv2$ any unit vector orthogonal
to it ($n-2$), and so on, giving $(n-1) + (n-2) + \dots + 0 = n(n-1)/2$. (For
$n=2$ this is $1$, a single angle; for $n=3$ it is $3$, the Euler angles.)
Requiring the $n$ images to be pairwise orthogonal imposes one equation per
pair, again $n(n-1)/2$. Constraints exactly consume parameters, so the solution
set is generically discrete --- finitely many sign and permutation variants of
one frame, not a continuum. Equivalently, $A\T A = V\Sigma'^2 V\T$ forces $V$ to
be an eigenbasis of the symmetric $A\T A$, and eigenbases are rigid.

\paragraph{The one exception.} When $A\T A = \sigma^2 I$ --- $A$ a scalar
multiple of an orthogonal matrix --- every direction stretches equally, every
orthonormal $V$ works, and $U = AV/\sigma$. The image of the sphere is a sphere,
so no direction is distinguished. This is worth stating as a principle:
\emph{the content of the SVD is anisotropy.} The decomposition earns its keep
exactly when the singular values differ; the ratio $\sigma_1/\sigma_n$ is the
eccentricity of the image ellipsoid, and also the condition number.

\subsection*{Exactly what freedom remains}

So $V$ is forced --- but not \emph{completely}. The singular values are rigid
(with the descending convention $\sigma_1 \geq \dots \geq \sigma_r > 0$ they are
uniquely determined by $A$), yet the frames retain a sliver of freedom, and it
is worth naming precisely what:

\begin{enumerate}\itemsep0.3em
\item \textbf{Sign.} For a simple singular value, $\uu i \to -\uu i$ and
$\vv i \to -\vv i$ together. The coupling is essential: you cannot flip one
without the other, since $A = U\Sigma V\T$ must be preserved.
\item \textbf{Repeated singular values.} If $\sigma_i$ has multiplicity $k$,
the corresponding $k$-dimensional subspaces of $U$ and $V$ can be rotated by
any $k \times k$ orthogonal matrix --- applied to \emph{both}. (This is the
exception above, localized to one repeated value instead of all of them.)
\item \textbf{Zero singular values.} The columns of $U$ spanning
$\operatorname{coker}(A)$ and of $V$ spanning $\ker(A)$ are arbitrary
orthonormal bases of those subspaces, independent of each other.
\item \textbf{Ordering}, if the descending convention is dropped.
\end{enumerate}

\begin{claim}
For $A$ of full rank with distinct singular values, the (thin) SVD is unique up
to the signs in item~1.
\end{claim}

\noindent
That is the whole of it: the stretches are pinned exactly, the frames up to
signs, with extra room only where a singular value repeats or vanishes.

\paragraph{A reparametrization, not a compression.} Stepping back, the SVD of an
$n \times n$ matrix supplies
\begin{equation}\underbrace{\tfrac{n(n-1)}{2}}_{V} + \underbrace{n}_{\Sigma} + \underbrace{\tfrac{n(n-1)}{2}}_{U} \;=\; n^2
\end{equation}
numbers --- exactly the $n^2$ entries it started with. The information is
rearranged from ``what each column does'' into ``(input frame, stretches, output
frame)'', and that arrangement is where the geometry lives.

\paragraph{$U$ and $V$ are independent.} One consequence of the count deserves
emphasis. Pick any orthogonal $U$, any orthogonal $V$, any nonnegative diagonal
$\Sigma$: then $A = U\Sigma V\T$ is a matrix with exactly that SVD. Nothing
couples the two frames. Going \emph{backwards} from a given $A$, they are
determined but generally unrelated --- you cannot infer one from the other
without passing through $A$. This is precisely why the SVD carries more
information than a single eigenbasis: two independent frames rather than one.

\section{Relation to eigenvalues and eigenvectors}
\label{sec:eigen}

For a general square matrix, the two decompositions are largely unrelated.

\begin{equation}A = \begin{pmatrix} 0 & 1 \\ 0 & 0\end{pmatrix}:
\quad \lambda = 0, 0 \quad\text{but}\quad \sigma = 1, 0 .
\end{equation}
\begin{equation}A = \begin{pmatrix} 1 & t \\ 0 & 1\end{pmatrix}:
\quad \lambda = 1, 1 \ \text{ for all } t, \quad\text{but}\quad \sigma_1 \to \infty \ \text{ as } t \to \infty .
\end{equation}

The structural difference: singular vectors are always orthonormal and always
exist in full; eigenvectors need not be orthogonal, and a defective matrix does
not have a complete set. Eigenvectors live in one space and map to themselves;
singular vectors come in \emph{coupled pairs across two spaces}. That is the
fundamental distinction --- eigendecomposition uses one basis, SVD pairs two.

\paragraph{What does hold.}
\begin{itemize}\itemsep0.3em
\item $\prod |\lambda_i| = \prod \sigma_i = |\det A|$.
\item $\sigma_{\min} \leq |\lambda_i| \leq \sigma_{\max}$ for every eigenvalue;
in particular the spectral radius satisfies $\rho(A) \leq \sigma_{\max} = \|A\|_2$.
\item Weyl: $\prod_{i=1}^k |\lambda_i| \leq \prod_{i=1}^k \sigma_i$ for each $k$,
with equality at $k = n$.
\item Yamamoto: $\lim_{k\to\infty} \sigma_i(A^k)^{1/k} = |\lambda_i|$.
\end{itemize}

\paragraph{When they coincide.}
\begin{itemize}\itemsep0.3em
\item \textbf{Normal} $A$ (i.e.\ $A\T A = AA\T$): $\sigma_i = |\lambda_i|$, and
the eigenbasis is orthonormal, so $\vv i$ are the eigenvectors and
$\uu i = \pm \vv i$.
\item \textbf{Symmetric positive semidefinite}: $\sigma_i = \lambda_i$ and
$U = V$. The two decompositions are literally identical.
\item \textbf{Symmetric indefinite}: $\sigma_i = |\lambda_i|$, and $U$ and $V$
differ by sign flips on the negative eigendirections.
\end{itemize}

\noindent
So $U = V$ characterizes symmetric positive semidefinite matrices. Neither
example in this note is symmetric, which is why neither has aligned frames.

A caveat on the normal case: eigenvalues $\pm 3$ in a symmetric matrix give a
\emph{repeated} singular value $\sigma = 3$, so the SVD acquires rotation
freedom in that subspace that the eigendecomposition does not have. The
eigendecomposition can be strictly more rigid.

\section{Polar decomposition}
\label{sec:polar}

The gap between the input frame $V$ and the output frame $U$ has come up more
than once. Figure~\ref{fig:ellipse-detail} showed it as a concrete angle --- the
$23.20^\circ$ by which the running example's frames differ --- and
\S\ref{sec:choose} made the general point that $U$ and $V$ are independent, so
this gap is real information, not an artifact. It is worth isolating that gap as
a single object, and the SVD does it with one regrouping.

Start from $A = U\Sigma V\T$ and insert $V\T V = I$ between two copies of the
frames:
\begin{equation}
A = U\Sigma V\T = U\,(V\T V)\,\Sigma V\T = (U V\T)(V\Sigma V\T) .
\end{equation}
Name the two factors $Q = UV\T$ and $P = V\Sigma V\T$. Then
\begin{equation}
A = QP,
\end{equation}
the \textbf{polar decomposition}. Each factor is exactly what its construction
makes it: $Q = UV\T$ is orthogonal (a product of orthogonal matrices), and it is
precisely the rotation that carries $V$'s frame onto $U$'s --- the gap itself,
named. $P = V\Sigma V\T$ is symmetric positive semidefinite, a pure stretch
along $V$'s axes by the factors $\sigma_i$, with no rotation left in it.

So every matrix is a stretch followed by a rotation: $P$ deforms the input along
the right singular directions, and $Q$ then swings the result over to the output
frame. It is the same content as the SVD --- $Q$ and $P$ are built from
$U, \Sigma, V$ --- repackaged to put the rotation in one box and the stretch in
another. For the running example $Q$ is the rotation by $23.20^\circ$ seen in
Figure~\ref{fig:ellipse-detail}; for the shear it is $26.57^\circ$. The analogy
with a complex number $z = re^{i\theta}$ --- modulus $r \geq 0$ times a unit
phase --- is exact and gives the decomposition its name: $P$ is the modulus,
$Q$ the phase.

\part{Where the decomposition goes}

\noindent
Part I built the singular value decomposition from the ground up and stopped at
the moment it was proved to exist. This part spends it. The SVD --- or its
symmetric shadow, the eigendecomposition --- sits underneath a surprising share
of modern data analysis and machine learning, and the sections below trace five
of those connections. Each states plainly what the decomposition \emph{does}
supply and, just as plainly, where its contribution stops and some other idea
takes over; the honesty is part of the point, because ``the SVD is everywhere''
is less useful than knowing exactly which part of a method it is responsible for.

One thing is worth watching for as they go by. In several of these connections it
is not the \emph{result} of a Part~I proof that reappears but the proof
\emph{itself} --- the argument, run as a procedure. The construction that found
the frame becomes an algorithm that computes it; the lemma that characterized the
maximizer becomes the rule that tells an optimizer when to stop. Where that
happens it is flagged, and left to speak for itself.

\section{Where this leads: principal component analysis}
\label{sec:pca}

If you have ever run principal component analysis, reduced the dimension of a
dataset, or called a low-rank approximation, you have used the object this
article has been building --- though probably without the picture. Here is the
connection, in one step.

Let $A$ be a data matrix: $n$ rows, one per observation, and $p$ columns, one
per feature, with each column \emph{centered} to have mean zero. Then the
$p \times p$ matrix
\begin{equation}\frac{1}{n-1}\,A\T A
\end{equation}
is exactly the sample covariance matrix of the features --- its $(i,j)$ entry is
the covariance of feature $i$ with feature $j$, because centering makes the
column inner products into covariances. But $A\T A$ is the very matrix whose
eigenvectors are the right singular vectors of $A$ (\S\ref{sec:algebraic}). So
the right singular vectors $\vv1, \vv2, \dots$ are the eigenvectors of the
covariance --- the \textbf{principal components}, the orthogonal directions of
greatest variance in the data. And the singular values give the variances
outright: the spread of the data along $\vv i$ is
\begin{equation}\frac{\sigma_i^2}{n-1},
\end{equation}
so the largest singular value marks the direction of most variance, the next the
most among directions orthogonal to it, and so on --- which is precisely the
greedy description of the frame from \S\ref{sec:variational}, now read as a
statement about data.

The geometry carries over intact. The ellipsoid this article has drawn since
\S\ref{sec:observe} --- the image of the unit sphere under $A$ --- is, for a data
matrix, the shape of the data cloud itself, and its axes are the principal axes.
Finding the perpendicular frame that maps to those axes, the whole problem the
article set out to solve, \emph{is} finding the principal components. PCA is not
an application of the SVD so much as the same picture wearing the clothes of
statistics.

\section{Kernel methods: the two Gram matrices}
\label{sec:kernel}

\S\ref{sec:duality} noticed that a matrix $A$ carries two symmetric
companions: $A\T A$, living in the input space, and $AA\T$, living in the output
space, sharing their nonzero eigenvalues $\sigma_i^2$. That duality is not a
curiosity --- it is the engine of the \emph{kernel trick}, the idea that lets a
method reach into a vast feature space while computing only in a small one.

Suppose each data point $\mathbf{x}$ is lifted into a high-dimensional feature
vector $\phi(\mathbf{x})$, and stack these as the rows of a matrix $\Phi$: $n$
points (rows), $d$ features (columns), with $d$ possibly enormous. Principal
component analysis in feature space wants the eigenvectors of the $d \times d$
matrix $\Phi\T\Phi$ --- unusable if $d$ is a million, or infinite. But the
duality says the eigenvalues are shared with the $n \times n$ matrix
\begin{equation}K = \Phi\Phi\T, \qquad K_{ij} = \ip{\phi(\mathbf{x}_i)}{\phi(\mathbf{x}_j)},
\end{equation}
whose entries are just the pairwise inner products of the lifted points, and
whose eigenvectors map to the feature-space ones through $\Phi$. So one
diagonalizes the $n \times n$ matrix instead of the $d \times d$ one --- kernel
PCA --- and never touches the giant feature space directly. When there are fewer
samples than features, this is the difference between possible and impossible.

Notice what kernel PCA is really doing: it is running the \emph{proof} of the
duality, not merely quoting its conclusion. \S\ref{sec:duality} did not just
observe that $\Phi\T\Phi$ and $\Phi\Phi\T$ share a spectrum --- it exhibited the
map between their eigenframes, carrying an eigenvector of one to an eigenvector
of the other by applying $\Phi$ (there, $A$) and rescaling. Kernel PCA executes
exactly that transport: it finds the eigenvectors of the small $K$ and pushes
them through $\Phi$ to recover the feature-space principal components it could
never have computed directly. The argument that proved the two matrices are two
ends of one map is, step for step, the algorithm.

\paragraph{Where the SVD's contribution stops.} The duality explains why working
with the small matrix $K$ \emph{suffices} --- same spectrum, recoverable
eigenvectors. It does \emph{not} explain the other half of the trick: that the
inner product $\ip{\phi(\mathbf{x}_i)}{\phi(\mathbf{x}_j)}$ can often be computed
by a cheap formula $k(\mathbf{x}_i, \mathbf{x}_j)$ without ever forming
$\phi$ --- the Gaussian kernel evaluates an inner product in an
infinite-dimensional space in one line. That the cheap formula corresponds to a
genuine feature space is a separate theorem (Mercer's, on positive-definite
kernels), and it belongs to analysis, not to the linear algebra of Part~I. The
SVD supplies the dimension-dodge; Mercer supplies the cheap evaluation. Both are
needed, and they are different ideas.

\section{Gradient descent and the power method}
\label{sec:gradient}

Part~I found the top singular direction by \emph{maximizing} $\|A\mathbf{v}\|$
over the unit sphere, and derived the optimality condition
$A\T A\, \vv1 = \sigma_1^2 \vv1$. Read that condition again with an optimizer's
eye. The gradient of $\tfrac12\|A\mathbf{v}\|^2$ is $A\T A\, \mathbf{v}$, and the
eigenvector equation says exactly that this gradient points \emph{along}
$\mathbf{v}$ --- it has no component tangent to the sphere. There is nowhere left
to climb. \emph{The eigenvector equation is the stopping condition of gradient
ascent}, and Part~I's construction was, without saying so, a description of where
an optimizer comes to rest.

And it is the same \emph{proof}, not just the same statement. The heart of
Part~I's variational step was the lemma that at the maximizer,
$\ip{A\vv1}{A\mathbf{w}} = 0$ for every $\mathbf{w} \perp \vv1$ --- the images of
all competing directions are orthogonal to $A\vv1$. That is precisely the
assertion that the gradient $A\T A\,\vv1$ has no component along any tangent
direction $\mathbf{w}$: the lemma's ``no direction improves the stretch'' and the
optimizer's ``the gradient is tangentially zero'' are one computation written
twice. So an optimizer descending this landscape is not merely finding the same
answer as Part~I --- it is re-running the maximality lemma, its every step of not
being able to improve corresponding to a step of the proof.

This is not merely an analogy; it is how large singular value decompositions are
actually computed. Rather than form $A\T A$ and diagonalize it --- impossible for
a web-scale matrix --- one \emph{climbs}: start from a random $\mathbf{v}$,
repeatedly apply $A\T A$ and renormalize (the \emph{power method}), and the
iterate converges to $\vv1$. Peel it off and repeat for $\vv2$, exactly Part~I's
maximize-then-recurse recursion, now run numerically. The neural version,
\emph{Oja's rule}, is stochastic gradient ascent on the same objective and was
one of the early bridges between optimization and unsupervised learning: a single
neuron trained by it extracts the first principal component.

\paragraph{Where the SVD's contribution stops.} That gradient ascent
\emph{converges to the right answer} is not automatic. Maximizing $\|A\mathbf{v}\|$
on the sphere is a non-convex problem with a critical point at every eigenvector;
the method finds the global maximum only because the landscape is benign --- the
non-maximal critical points are saddles the iteration slides off, a property that
has to be proved, not assumed. And the fashionable claim that ``neural networks
do SVD by gradient descent'' should be handled with care: it is a precise theorem
for deep \emph{linear} networks, whose training dynamics learn the singular
directions of the input--output correlation in order of decreasing $\sigma_i$
(echoing Part~I's descending recursion), but it is not a statement about
nonlinear networks in general. The clean connection is to the optimization
landscape and the linear case; beyond that it is inspiration, not theorem.

\section{Low-rank approximation: keeping the largest \texorpdfstring{$\sigma_i$}{sigma}}
\label{sec:lowrank}

Truncate the decomposition: keep the $k$ largest singular values and their
vectors, discard the rest. The result $A_k = \sum_{i=1}^k \sigma_i \uu i \vv i\T$
is a rank-$k$ matrix, and the remarkable fact --- the \emph{Eckart--Young
theorem} --- is that it is the \emph{best} rank-$k$ approximation of $A$ there is:
no other rank-$k$ matrix comes closer, in either the spectral or the Frobenius
norm. The error is exactly the largest discarded singular value,
\begin{equation}\min_{\operatorname{rank}(B) = k} \|A - B\|_2 = \|A - A_k\|_2 = \sigma_{k+1}.
\end{equation}
This is Part~I's geometry read as compression: the singular values, sorted, rank
the directions of $A$ by importance, so keeping the top few keeps as much of $A$
as any rank-$k$ object can. It is the mathematical core of image compression,
noise filtering (small $\sigma_i$ are often noise), latent semantic analysis,
recommender-system factorization (users and items as low-rank factors), and the
low-rank adapters now used to fine-tune large models. In each, the SVD does not
merely \emph{offer} a compression --- it offers the provably optimal one.

\paragraph{Where the SVD's contribution stops.} Optimality is with respect to a
particular measure of error --- the spectral or Frobenius norm, which weight all
entries alike. When the application cares about something else (missing entries,
non-negativity, sparsity, a probabilistic likelihood), the plain SVD truncation
is no longer the right answer, and its relatives --- non-negative matrix
factorization, probabilistic PCA, matrix completion --- take over. The SVD is
optimal for the question ``closest rank-$k$ matrix in Euclidean norm,'' and
knowing that is knowing exactly when to reach for it and when not to.

\section{Eigenvectors of a graph: PageRank and spectral clustering}
\label{sec:graph}

The same machinery that finds the axes of an ellipsoid finds structure in a
network, once the network is written as a matrix. Two examples, both resting on
an eigenvector of a graph.

\emph{PageRank} ranks web pages by importance, where a page is important if
important pages link to it --- a circular definition that resolves into an
eigenvector equation. Encode the link structure as a column-stochastic matrix
$P$ (column $j$ spreads page $j$'s vote evenly over the pages it links to). The
importance vector $\mathbf{r}$ is the one unchanged by the process,
$P\mathbf{r} = \mathbf{r}$: the dominant eigenvector of $P$. And it is found by
precisely the power method of \S\ref{sec:gradient} --- start with uniform
importance, apply $P$ repeatedly, converge. PageRank is the maximize-and-settle
idea of Part~I, run on the web.

\emph{Spectral clustering} splits data into groups by building a graph of
similarities, forming its \emph{Laplacian} $L = D - W$ (degrees minus
weights), and reading off the eigenvector of $L$ with the second-smallest
eigenvalue --- the \emph{Fiedler vector}. Its signs partition the graph into two
well-separated pieces. The eigenvector exists, is orthogonal to the trivial
all-ones eigenvector, and extremizes a meaningful quantity (it minimizes a
relaxed cut) for exactly the reasons Part~I laid out: it is the extremal
direction of a symmetric matrix.

\paragraph{Where the SVD's contribution stops.} The spectral machinery guarantees
that the eigenvectors exist, are orthogonal, and solve an extremal problem --- the
content of Part~I, applied to a graph matrix. What it does \emph{not} decide is
the modeling step that came first: how to turn pages into a stochastic matrix, or
similarities into a Laplacian. Those choices --- which similarity, which
normalization, how many clusters --- are where the domain judgment lives, and the
eigenvector computation is only as good as the matrix it is handed. The SVD
supplies the ``how''; the graph construction is the ``what,'' and it is not the
SVD's to give.

\subsection*{A pattern in what recurred}

Look back at what carried over from Part~I. It was rarely a bare fact. The
maximize-then-recurse construction reappeared as the power method and as
PageRank; the lemma that the maximizer's gradient vanishes tangentially
reappeared as an optimizer's stopping rule; the proof that two Gram matrices are
two ends of one map reappeared as the transport at the center of kernel PCA. In
each case the object that transferred was an \emph{argument} --- a sequence of
steps first walked to establish that something is true --- now walked again to
compute it, or to run it. Not every proof came along: the sign-flip that opened
Part~I, the prettiest argument in it, has no descendant here; it is a way of
\emph{seeing} that the frame exists, not a way of \emph{finding} it. What the
recurrences have in common, and what that one lacks, is left for the reader to
weigh.

\section{Summary}

\begin{quote}
$A$ is completely described by which orthonormal input frame goes to which
orthonormal output frame, and by how much each direction stretches. The SVD's
claim is that such frames always exist.
\end{quote}

\noindent
Linearity means any basis determines $A$ --- that is not what makes the SVD
special. What makes the $V$-basis special is that in it the action is
\emph{diagonal}: feed in $\vv3$, get out pure $\uu3$, scaled. Feed in any other
basis and the images tangle together in a way that hides the geometry. Same
information, worse presentation.

That was Part~I: a frame found by watching an ellipse, a theorem earned rather
than assumed. Part~II spent it, and the spending revealed the article's real
subject. The maximize-and-recurse construction was not left behind on the page
once the SVD was proved --- it walked off to become the power method, PageRank,
the training of a principal-component neuron. The little lemma about where a
stretch is maximal became the condition an optimizer halts on. The observation
that $A\T A$ and $AA\T$ are two ends of one map became the trick that lets a
kernel method reach into a space it cannot afford to enter. What looked like an
island of geometry --- circles, ellipses, a rotating frame --- turned out to be
the mainland: the structure quietly running tools that a working data scientist
uses without a second thought.

The lesson underneath, never argued but repeatedly shown, is that a proof is not
only a certificate that something is true. It is a \emph{procedure} --- a route
walked once to establish a fact, and available to be walked again to compute it.
Understand the route, and you understand the algorithm; and the same
understanding recognizes it the next time it appears wearing different clothes.
That is the case for meeting a piece of mathematics like this one all the way
down to its proofs, rather than importing it as a name and a library call.

\section{Sources and prior art}
\label{sec:sources}

None of the mathematics here is new. The contribution, if any, is one of
arrangement: a path of discovery assembled from standard pieces, with the
powerful theorem placed at the end as a name for what the geometry already built
rather than as an assumption at the start. This section says plainly where each
piece comes from.

\paragraph{The classical results.} The picture of a matrix carrying the unit
sphere to an ellipsoid, with the singular values as semi-axis lengths, is the
standard geometric introduction to the SVD; it is the opening of Trefethen and
Bau's \emph{Numerical Linear Algebra}~\cite{trefethen1997}, whose Lecture~4 this
article's \S\ref{sec:ellipsoid} follows. The variational construction of
\S\ref{sec:variational} --- maximize $\|Av\|$, then recurse on the orthogonal
complement --- is also theirs, and is the proof of existence given in most
graduate texts. The route through the spectral theorem applied to $A\T A$
(\S\ref{sec:algebraic}), and the polar decomposition read off from it
(\S\ref{sec:polar}), are presented exactly this way in Axler's \emph{Linear
Algebra Done Right}~\cite{axler2024}, where the SVD appears as a consequence of
the spectral theorem and leads directly to the polar form. The computational
origin of the subject --- how one actually calculates an SVD --- is Golub and
Kahan~\cite{golub1965}; the early history is surveyed by
Stewart~\cite{stewart1993}. The four-fundamental-subspaces framing used in
\S\ref{sec:ellipsoid} is Strang's~\cite{strang2019}.

\paragraph{The closest relative.} The organizing conceit of this article ---
adjust a perpendicular frame until its images are also perpendicular, and that
configuration is the SVD --- appears, in interactive form, in Heath's
\emph{Interactive Educational Modules in Scientific
Computing}~\cite{heath2018, heath-iem}. That module places candidate singular
vectors on the image ellipse and lets the reader drag them until the frame is
orthogonal in both spaces at once, displaying the resulting $U$, $\Sigma$, $V$.
It is the same equivalence, driven from the image side rather than the domain
side, and it stops at two-dimensional visualization: it lets one \emph{find} the
surviving frame but does not argue that one must exist, and does not carry the
construction to higher dimensions. The existence argument of \S\ref{sec:existence}
and the recursion of \S\ref{sec:variational} are what this article adds on top of
that shared starting point.

\paragraph{The gap this fills.} Tomasi's widely used
notes~\cite{tomasi} state the central observation directly --- the two pairs of
points on the unit circle mapping to the ends of the ellipse's axes ``are always
orthogonal'' --- and then remark that ``simple and fundamental as this geometric
fact may be, its proof by geometric means is cumbersome,'' and prove it
algebraically instead. This article is, in effect, the geometric proof Tomasi
declines to give: the sign-flip argument of \S\ref{sec:planar} is the missing
elementary reason the axis-preimages are perpendicular, and the identity
$dg/d\theta = 2f$ of \S\ref{sec:variational} is what ties that perpendicularity
to the extremal stretch that the algebraic proofs start from. Neither is deep,
and neither is new in substance; but stating them, and staging the whole
development so the spectral theorem is earned rather than assumed, is the point
of the exercise.

This is not a survey. The literature on the SVD is vast, and the references
above are entry points chosen for their relevance to the particular path taken
here, not a representative sample of the field.

\bigskip
\hrule
\bigskip

\noindent
\footnotesize
All numerical values in this document were computed with NumPy's
\texttt{linalg.svd} and verified against the defining properties
($\ip{A\vv1}{A\vv2} = 0$, $\prod\sigma_i = |\det A|$, $\det(UV\T) = +1$).

\end{document}